\documentclass[1p]{elsarticle}

\pdfpagewidth=\paperwidth
\pdfpageheight=\paperheight

\usepackage{amsmath,graphicx}
\usepackage{verbatim}
\usepackage{color}
\usepackage{amsfonts,amsbsy,amssymb}
\usepackage{natmove}

\usepackage[scr=boondox]{mathalfa}
\def\br{\boldsymbol {\rm r}}

\def\bx{{\boldsymbol {\rm x}}}

\def\I{\mathcal I}
\def\d{\,\,{\rm d}}
\def\br{\boldsymbol {\rm r}}
\def\T{{\rm T}}
\def\E{{\rm E}}

\def\rhom{\rho_{\max}}

\def\bt{{\boldsymbol {\rm q}}}
\def\bq{\bt}
\def\hatq{{\boldsymbol {\hat {\rm q}}}}
\def\hatQ{{\boldsymbol {\hat {\bQ}}}}
\def\bQ{{\boldsymbol{\mathscr{q}}}}
\def\A{{\boldsymbol {\rm A}}}
\def\P{{\boldsymbol {\rm P}}}
\def\I{{\boldsymbol {\rm I}}}
\def\tq{{\boldsymbol {\tilde {\rm q}}}}

\newtheorem{lemma}{Lemma}
\newtheorem{definition}{Definition}
\newtheorem{corollary}{Corollary}
\newdefinition{remark}{Remark}
\newproof{proof}{Proof}
\newproof{example}{Example}

\usepackage{calligra}

\DeclareMathAlphabet{\mathcalligra}{T1}{calligra}{m}{n}
\DeclareFontShape{T1}{calligra}{m}{n}{<->s*[2.2]callig15}{}
\newcommand{\sr}{\mathcalligra{r}\,}
\newcommand{\dint}[2]{\int_{#1} \!\!\! \d \br \int_{#2} \!\!\! \d \bx \,\,}
\newcommand{\np}[1]{\mathcal N(\bx-\br) p_{#1}(\bx)}

\begin{document}

\author{Paul N.\ Patrone and Anthony J.\ Kearsley}
\ead{paul.patrone@nist.gov, anthony.kearsley@nist.gov}
\address{National Institute of Standards and Technology}

\date{\today}

\title{Inequalities for Optimization of Classification Algorithms: \\ A Perspective Motivated by Diagnostic Testing}

\begin{abstract}
Motivated by canonical problems in medical diagnostics, we propose and study properties of an objective function that uniformly bounds uncertainties in quantities of interest extracted from classifiers and related data analysis tools.  We begin by adopting a set-theoretic perspective to show how two main tasks in diagnostics -- classification and prevalence estimation -- can be recast in terms of a \textcolor{black}{variation on the confusion (or error) matrix $\P$ typically considered in supervised learning.}  We then combine arguments from conditional probability with the Gershgorin circle theorem to demonstrate that the largest Gershgorin radius $\rhom$ of the matrix $\I-\P$ (where $\I$ is the identity) yields uniform error bounds for both classification and prevalence estimation.  In a two-class setting, $\rhom$ is  minimized via a measure-theoretic ``water-leveling'' argument that optimizes an appropriately defined partition $U$ generating the matrix $\P$.  We also consider an example that illustrates the difficulty of generalizing the binary solution to a multi-class setting and deduce relevant properties of the confusion matrix.
\end{abstract}

\maketitle

\section{Introduction}

In early 2020, the emergence of the SARS-CoV-2 virus created a need for medical diagnostic tests that could monitor the spread of COVID-19.  In response, laboratories across the world developed hundreds of serology assays to detect immune markers of past infection \cite{assaynumber,EUA}.  Despite this global effort, tasks such as assay {\it optimization} were often addressed in a qualitative or even ad-hoc manner, and by the middle of 2020, the Food and Drug Administration (FDA) had already revoked the emergency use authorizations of several tests \cite{revoked}.  At the same time, there was growing awareness that {\it interpretation} of diagnostic results remained a challenging problem due to significant sources of uncertainty and unresolved questions of how to best estimate disease prevalence \cite{controversial,flawed}.  While the historical relationship between these  events may never be fully understood, their coincidence begs the question: how could methods for analyzing diagnostic data have been used to optimize the assays themselves?  

In the intervening years, it has become clear that new perspectives on the mathematics of classification theory are necessary to address such questions.  While the medical and public-health communities have always leveraged basic elements of probability \cite{Medstat1,Lewis12}, they have often remained tied to binary diagnostic settings \cite{Lew89,Lang14,Lewis12}, one-dimensional measurements \cite{OldPrevOpt}, and restrictive modeling choices \cite{OldPrevOpt,Lewis12}, typically with limited uncertainty quantification (UQ).  See also the introduction of Ref.\ \cite{Qiu19} for a related overview.  In contrast, we recently showed that deeper concepts arising from the intersection of measure theory, mathematical analysis, optimization, and metrology enable generalizations beyond the status quo \cite{Patrone21,Patrone22,Patrone22_1,Luke22,Luke23,Bedekar22,PartI}.  This new perspective led to several results that were, to the best of our knowledge, unknown to the public health, epidemiology, \textcolor{black}{and even machine-learning} (ML) communities: (i) prevalence, which often plays the role of a prior in generative classification, has unbiased and converging estimators in arbitrary dimension, and thus need not be modeled \cite{Luke22}; (ii) minimum-uncertainty prevalence estimates do not require classification \cite{Patrone21,Luke23,Bedekar22,PartI};\footnote{In the contexts motivating this manuscript, prevalence is usually understood to mean seroprevalence, or the fraction of individuals that express an immune response as quantified in terms of certain antibodies that attach to a virus (for example).  This is typically taken as a proxy for number of individuals that have been infected.} and (iii) discriminative and generative classifiers can be unified via a level-set theory that establishes an isomorphism between prevalence and the underlying conditional probability distributions \cite{PartI}.  Critically, these works pointed to the central, if not underappreciated, role of prevalence in classification.  At the same time, they only considered data analysis in the context of a fixed input space, but not the reverse question of how such tasks inform assay design, and more generally, selection of the input space itself.

The present manuscript aims to complete this virtuous circle by studying  uniform bounds on properties typically extracted from classifiers.  We begin by observing that both classification and prevalence estimation can be recast in terms of a \textcolor{black}{(non-standard)} confusion matrix $\P$ whose elements quantify the overlap between class-conditional probability density functions (PDFs) on the input space $\Gamma$.  We then demonstrate that uncertainty in these tasks is bounded from above in terms of the largest Gershgorin radius $\rho_{\max}$ of $\I-\P$, where $\I$ is the identity matrix.   In the binary setting, we also determine the matrix $\P$ corresponding to the smallest $\rhom$, which yields least upper bounds on assay uncertainty.  Interestingly, the solution to this problem arises from a ``water-leveling'' argument applied to the class-conditional PDFs.  This construction motivates an example that illustrates the challenges of minimizing $\rhom$ in multi-class settings, which we partially address through a pair of lemmas that deduce additional properties of $\P$.  Finally, we conclude by (i) examining further mathematical properties of $\rhom$ that make it a practical tool for selecting (or even optimizing) the input space $\Gamma$, and (ii) considering inexpensive numerical approximations of $\rhom$.

The practical motivation for studying {\it uniform bounds} such as $\rhom$ arises from the fact that assays (and more generally classifiers) are often developed under circumstances differing from the ones to which they are applied.  The COVID-19 pandemic provides important examples: many assays developed in 2020, when seroprevalence was low, suffered notable changes in accuracy as infections become more widespread \cite{change}.  This was likely a result of the facts that (i) heuristcs are often used to select the antigens (i.e.\ the input space) in diagnostic tests and (ii) data analyses thereof do not account for variable prevalence values.  Thus, uncertainty bounds defined in terms of quantities such as $\rhom$ serve at least two purposes.  First, they quantify the worst-possible performance of a Bayes-optimal classifer under any prevalence, which can be used as a criterion for deciding when further optimization of $\Gamma$ is needed.  Second, they quantify the accuracy with which prevalence can be estimated and subsequently used to construct a Bayes optimal classifier.  This latter consideration in particlar distinguishes our perspective from previous works.  Many authors have studied the utility of confusion matrices and nonlinear functions thereof as objective functions for training classifiers; see Refs.\ \cite{Metrics1,Metrics2,confusion,Imbalanced,Imbalanced2,Metrics3,Metrics4,Metrics5,Metrics6} and the citations therein.  However, the role of prevalence and estimation thereof in remains, to the best of our knowledge, a largely unstudied topic.

An overarching technical theme of our analysis is the fact that multiple layers of probabilistic conditioning separate the information that diagnosticians want -- the numbers and identities of samples in each class -- from the measurements that can be made.  In particular, the prevalence $q_j$ of the $j$th class is a discrete probability density that is, in practice, the unknown quantity to be estimated. Each $q_j$ is paired with a class-conditional PDF $p_j(\br)$, which quantifies the probability of an element in class $j$ yielding the random variable $\br\in \Gamma$.  Through the law of total probability, the $p_j(\br)$ ``extend'' $q_j$ onto a (typically) continuous space, which is then projected onto a new, discrete density $\bQ=\P\bq$ defined in terms of arbitrary indicator functions (where $\bq=(q_1,q_2,...,q_c)$ for $c$ classes).  The freedom to define these indicator functions yields multiple interpretations of the assay, the choice of which affects all downstream data analyses.  Thus, informally speaking, the main tasks in diagnostics are to (1) define a ``desirable'' projection operator $\P$ (i.e.\ confusion matrix) through the selection of $\Gamma$ and (2) reverse the flow of information, going from $\bQ$ to $\bq$.  However, as our goal is essentially to bound uncertainty in step (2) given (1), we must carefully navigate these layers of conditioning in order to quantify their relative impacts.  

A key limitation of our work is that we assume exact knowledge of the $p_j(\br)$.  This is for several reasons.  First, an analysis of this situation is sufficiently rich that it warrants a study of its own.  Second, we anticipate that a full analysis of the uncertainty induced by not knowing the $p_j(\br)$ is difficult enough so as to be beyond the scope of this manuscript.  Here our goal is to determine guaranteed best bounds on uncertainty in prevalence estimates, which correspond to an infinite-data limit.  While this is clearly not achievable in practice, many machine-learning settings are sufficiently rich with data that we anticipate the results presented herein are accurate approximations.  

Finally, we note that while the work herein applies more generally to classification theory, the original impetus arose from medical diagnostics.  Thus, we frequently use tasks and concepts from this field to motivate and explain our results.

The remainder of this manuscript is organized as follows.  Section \ref{sec:background} provides mathematical context and key assumptions behind our analysis.  Section \ref{sec:bounds} derives uniform bounds on classification error and prevalence estimation.  Section \ref{sec:minrho} provides methods for further optimizing these bounds in a binary classification setting.  Finally, Sec.\ \ref{sec:discussion} discusses the challenges of extending the binary result, considers limitations, and surveys open directions.  
\section{Background, Definitions, Assumptions, and Perspectives}
\label{sec:background}

\subsection{Mathematical Context: Classification Theory as Motivated by Diagnostics}

Diagnostic and classification problems assume the existence of a discrete population $\Omega$ whose members or ``individuals'' $\omega$ belong to one of $c$ distinct classes $C_k$ ($k\in \{1,2,...,c\}$) according to some probability distribution.    This setting is rigorously described in terms of an underlying probability space associated with the multinomial distribution \cite{multi,pspace}.  \textcolor{black}{For our purposes, it suffices to consider the prototypical case of a single sample point $\omega$.}  The corresponding probability space is constructed by considering: (i) the set of classes (i.e.\ outcomes) $\mathcal C=\{C_k\}$ for an individual  $\omega$; (ii) the induced power-set $2^{\mathcal C}$ of possible events; and (iii) the discrete probability $\bt=(q_1,q_2,...,q_c)$ of belonging to a class, which induces the measure $\mu_\bt$ associated with each event in $2^{\mathcal C}$.  Then the triple $(\mathcal C, 2^{\mathcal C},\mu_\bt)$ is the probability space associated with $\omega$.  In this context, $\bt$ is often an unknown quantity to be estimated, and it plays an elevated role contexts such as epidemiology.
\begin{definition}
Assume that a population has $c$ classes, and let $\bt=(q_1,q_2,...,q_c)$ be a discrete probability density, where $q_k$ is the probability of a sample from the population belonging to class $C_k$.  We refer to $\bt$ as the {\bf prevalence}.
\end{definition}
\begin{remark}
We distinguish the classes $C_k$ from their indices $k$.  While we could simply order the classes $k=1,2,...,c$, the symbol $C_k$ reminds us that each class has an associated label, such as positive or negative.
\end{remark}

In practice, we are not given the true class $C_k$ of an individual, but instead a random variable $\br(\omega)$ from which we infer this information.  The $\br(\omega)$ belong to a set or input space $\Gamma \subset \mathbb R^n$, where $n$ is a positive integer on the order of 20 or less in typical diagnostics problems; see Refs.\  \cite{dim1,dim2,dim3,dim4,dim5,dim6,dim7} for characteristic examples.  For the purposes of illustration, we may think of these vectors as being generated by an instrument that measures properties of an individual $\omega$.  Each dimension of $\br$ can be associated with a different measurement ``target,'' such as the quantity of a given antibody type in a blood sample.  Because the remainder of the manuscript focuses primarily on $\br$, we omit its dependence on $\omega$.  However, the reader should remain aware that $\br$ is a mapping from the population $\Omega$ to the measurement or input space $\Gamma$.    Note also that the random variable $\br(\omega)$ broadens the interpretation of $\omega$ to include not only the class of an individual, but also their ``antibody levels.''  Thus, the $\sigma$-algebra and probability measure must be generalized accordingly, although such details are not relevant and are omitted.  We always assume that the underlying probability space exists and is well defined.  \textcolor{black}{In light of this expanded notion of a sample point, it is useful to define a function $C(\omega)$ that ``projects out'' the class of $\omega$; i.e.\ $C(\omega)=C_k$ if and only if $C_k$ is the class associated with $\omega$.}

\begin{remark}
In practice, one is typically given a collection of independent sample points $\omega_i$ to classify and from which to estimate the prevalence.  We denote independent realizations of the corresponding $\br$ by $\br_i=\br(\omega_i)$.  Unless otherwise noted, indexing of $\br$ does not refer to its components.  In an abuse of terminology, we also refer to the class of $\br$, despite $\omega$ being the object classified. 
\end{remark}

Given a collection $\mathcal S=\{\br_i\}$ (which may have repeated elements if two samples are sufficiently similar) \cite{multiset},\footnote{Thus, $\mathcal S$ is a multiset.} diagnosticians seek to answer the following questions:
\begin{itemize}
\item[(I)] What is the best way to assign each $\br_i$ to a unique class $C_k$?  More precisely, how do we define a function $\hat C: \Gamma \to \mathcal C$ such that we maximize the probability of $\hat C(\br(\omega)) = C(\omega)$ being true?
\item[(II)] How do we estimate the true prevalence $\bt$ of $\Omega$ given $\mathcal S$?
\end{itemize}   
The present manuscript supplements these with a third question, stated here informally:
\begin{itemize}
\item[(III)] Among all available assays (i.e.\ input spaces $\Gamma$), which is the best at classifying samples and estimating prevalence?
\end{itemize}

Previous work has shown that in order to answer Questions (I) and (II), it is useful to postulate the existence of probability density functions (PDFs) $p_k(\br)$ conditioned on a sample belonging to class $C_k$ \cite{Patrone21,Patrone22,Patrone22_1,Luke23}.\footnote{In typical serology settings, we may always assume knowledge of the $p_k(\br)$.  Assays are validated by collecting training data, which can be used to construct the PDFs.}  Using the law of total probability, we combine the prevalence and these PDFs to construct \cite{pspace}
\begin{align}
\mathcal Q(\br) = \sum_{k=1}^c q_k p_k(\br),
\end{align}
which is the probability density associated with $\br(\omega)$, i.e.\ the measurement outcome from a sample taken at random from the population.  We also require a partition $U=\{D_j\}$ of $\Gamma$, where each $D_j$ is associated with class $C_j$.  How we define this association depends on whether question (I) or (II) is at hand.

For classification, the sets $D_j$ are used to directly interpret the $\br_i$.  That is, we assign $\br$ to class $C_j$ if $\br\in D_j$\textcolor{black}{, in which case $\hat C(\br) = C_j$}.  Note that this assignment is a choice, not an objective statement of the sample's underlying true class, which is assumed to be unknown.  To answer Question I, we construct the partition $U^\star=\{D_j^\star\}$ that minimizes an appropriate notion of classification error.  Following Refs.\ \cite{Patrone21,Luke22}, we express this error as
\begin{align}
\mathcal E = 1- \sum_{j=1}^c \int_{D_j}   q_j p_j(\br) \d \br,  \label{eq:error}
\end{align}
which is the probability of classifying a sample incorrectly.  It is straightforward to show that $\mathcal E$ is minimized when $\hat C(\br)=C_j$ for the value of $j$ that maximizes the product $q_jp_j(\br)$ \cite{Luke22,Zhang,RW}, a fact that we will leverage in Sec.\ \ref{sec:discussion}.

For prevalence estimation, the association of $D_j$'s to their corresponding classes is more subtle and has not been fully explored \cite{Luke22}.  Indeed, a theme of the present manuscript is to unravel this connection by showing that despite being distinct tasks, classification and prevalence estimation are controlled by the same spectral properties of a matrix $\P$ quantifying the overlap between the conditional PDFs.  This motivates the following definition; cf.\ also Ref.\ \cite{confusion}.

\begin{definition}
Let $U=\{D_j\}$ be a partition of the input space $\Gamma$ and $p_j(\br)$ be conditional probability densities that map $\br\in\Gamma$ to $\mathbb R$.  We refer to the matrix $\P$ having elements
\begin{align}
P_{j,k}=P_{j,k}(U)=\int_{D_j} p_k(\br) \d \br,
\end{align}
as the {\bf confusion matrix induced by} $U$.  
\end{definition}

\begin{remark}
We assume absolute continuity of measure \cite{Tao}, given that we postulate the existence of the $p_j(r)$.  In particular, whenever the Lebesgue measure of a set $D$ is zero, then we assume that $\int_{D} p_k(\br) \d \br=0$ for all $k$.   We need this property to ensure that, given a set $D$ with measure $\mu_D$, we can always find a subset $D'\subset D$ with an arbitrary measure $\mu_{D'}<\mu_D$.
\end{remark}

The matrix $\P$ plays a fundamental role in our analysis and simplifies several formulas.  For example, the classification error can be expressed as 
\begin{align}
\mathcal E = {\rm Tr}[(\I-\P){\rm diag}(\bq) ], \label{eq:traceerror}
\end{align}
where ${\rm diag}(\bq)$ is the diagonal matrix whose $(k,k)$th entry is $q_k$, and ${\rm Tr}$ is the trace operator; see, e.g.\ Ref.\ \cite{Luke23}.  Since the prevalence $\bt$ (which we take to be a column vector) is a probability density, the product $\P\bt$ is the ``push-forward'' density associated with a new random vector $\hatQ$ \cite{Tao}; viz. $\bQ=\P\bt$ is the probability of $\hatQ$,  the fraction of samples falling in each element of $U$.  Given a realization $\hatQ$ and assuming $\P$ is invertible, one can define a prevalence estimate $\hatq=\P^{-1}\hatQ$.  Several works have explored properties of this estimate, albeit expressed in a slightly different form \cite{Patrone21,Luke22,Luke23}.  For example, it is straightforward to show that $\hatq$ so defined is unbiased and converges in mean-square as the number of samples $s\to \infty$.  Here the focus is on estimating the weighted uncertainty
\begin{align}
\sigma^2(\A)=\E\left\{(\hatq - \bt)^\T \A (\hatq - \bt) \right\} \label{eq:vardef}
\end{align}       
where $\E$ denotes expectation and $\A$ is an arbitrary positive semi-definite weighting matrix.  In light of Eq.\ \eqref{eq:vardef}, Question (II) can be restated as the task of identifying a partition $U$ that minimizes $\sigma^2(\A)$.

It is known that the partition minimizing $\sigma^2(\A)$ is in general different than the one that optimizes classification accuracy \cite{Patrone22_1}, and both depend explicitly on the prevalence.  Moreover, an analytic solution to Question (II) is not known to the authors when the number of classes is greater than two.  This leads us to reinterpret Question (III) as follows:
\begin{itemize}
\item[(IIIa)] Are $\mathcal E$ and $\sigma^2(\A)$ uniformly bounded from above in $\bt$ by some norm of $\P$?
\item[(IIIb)] What is the partition $U^\star$ that minimizes this norm?
\end{itemize}
It is desirable that the bound associated with Question (IIIa) be sharp.  In doing so, one characterizes an assay by its worst possible performance, both with respect to classification and prevalence estimation.  

\subsection{Key Properties and Assumptions of the Matrix $\P$}

To proceed further, it is necessary to establish certain facts and assumptions about $\P$.  The following definition informs Question (IIIa).
\begin{definition}
Let $\P$ be a confusion matrix.  We refer to
\begin{align}
\rhom=\max_{j}\{1-P_{j,j}\}
\end{align}
as the {\bf largest column-wise Gershgorin radius}.  More generally, we refer to $\rho_j=1-P_{j,j}$ as a {\bf column-wise Gershgorin radius.}  
\end{definition}

\begin{remark}
In the remainder of this manuscript, we omit the modifier ``column-wise,'' since it is clear from context that these are the only Gershgorin radii that we consider.  See Ref.\ \cite{Gershgorin1,Gershgorin2} for a more complete discussion on Gershgorin circles and their radii.  
\end{remark}

Next we deduce useful properties of the overlap matrix and impose needed structure.  
\begin{itemize}
\item[(P1)] First observe that $\P$  is a left-stochastic matrix with non-negative entries \cite{StochasticMatrix1}.  
\item[(P2)] As a result, the columns of $\I-\P$ sum to zero, which implies
\begin{align}
[\I-\P]_{k,k}=\sum_{j:j\ne k} P_{j,k}.
\end{align}
\item[(P3)] Further assume $U$ can be (and has been) chosen so that $\P$ is strictly diagonally dominant column-wise.  That is
\begin{align}
P_{k,k} > \sum_{j:j\ne k} P_{j,k}.
\end{align}  This is equivalent to $P_{k,k}>1/2$.
\item[(P4)] By the Gershgorin circle theorem \cite{Gershgorin1,Gershgorin2}, (P3) implies that the real part of the eigenvalues of $\P$ are positive definite; thus $\P$ is invertible.  
\item[(P5)] Moreover, the spectral radius $\mathfrak r$ of $\I-\P$ is bounded as $\mathfrak r \le 2\rhom$ \cite{grad}.

\item[(P6)] Similarly, the smallest eigenvalue $\lambda_{\min}$ {\it in magnitude} of $\P$ is bounded from below as $1-2\rhom\le |\lambda_{\min}|$ \cite{Gershgorin1,Gershgorin2}.
\end{itemize}

\begin{figure}
\begin{center}\includegraphics[width=12cm]{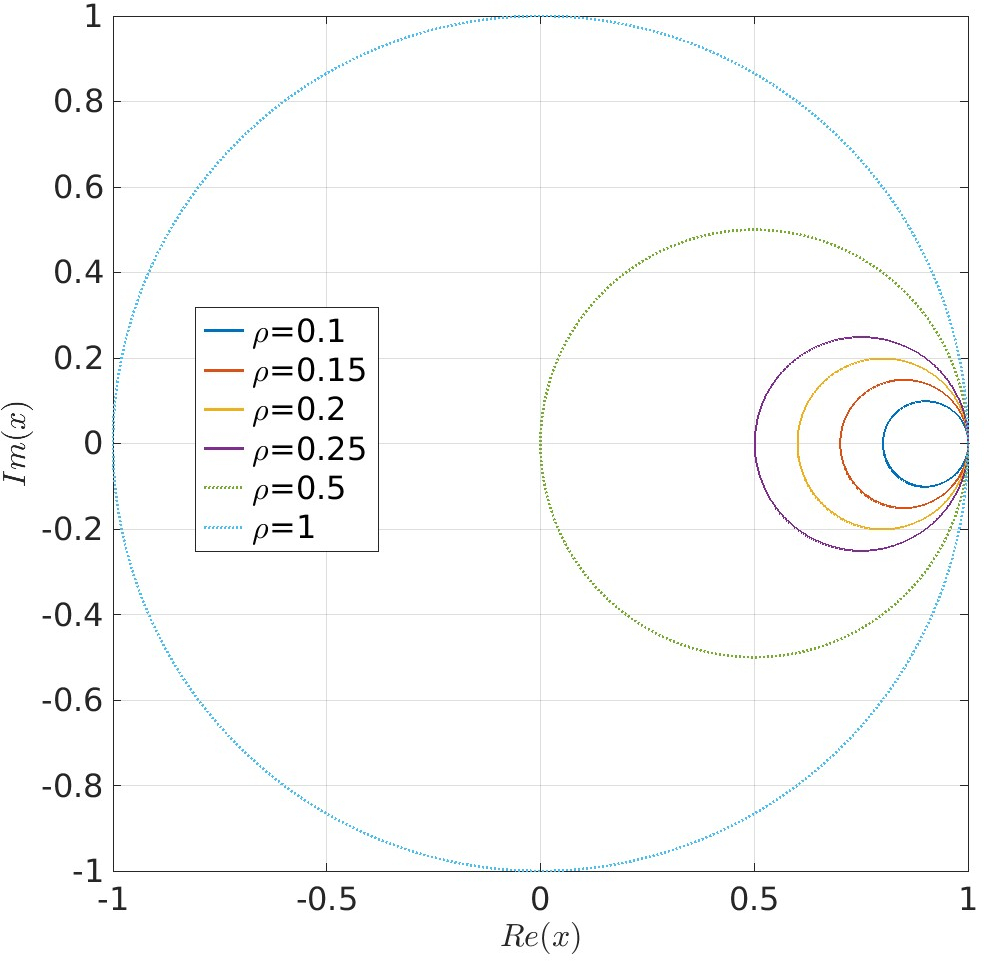}\end{center}\caption{Gershgorin circles associated with a matrix $\P$.  The Gershgorin radii $\rho$ are equal to the differences $1-P_{j,j}$ for all diagonal elements of $\P$.  Properties P1 and P2 imply that the circles are cotangent at the point $(1,0)$ in the complex plane.  Property P3 implies that the Gershgorin radii must satisfy $\rho < 0.5$.  In light of this, properties P4 -- P6 imply that the Gershgorin circles cannot touch or contain the origin.  For a viable assay, the corresponding $\P$ could not have the largest two Gershgorin circles as depicted above.  The case $\rho=0.5$ corresonds to a class for which a corresponding sample is classified incorrectly at least 50\% of the time, which is the hallmark of a bad diagnostic.  The case $\rho = 1$ corresponds to a class whose samples are always  mis-classified.      }\label{fig:gershgorins}
\end{figure}

In the following sections, we show how $\rhom$ yields bounds for both $\mathcal E$ and $\sigma^2(\A)$ and construct the partitions $U^\star$ that minimize $\rhom$ in certain cases.  However, several comments are in order.  

From the perspective of classification, (P3) implies that the probability of a sample from class $C_j$ falling into domain $D_j$ is greater than $1/2$; i.e.\ a sample will more likely than not be classified correctly.  As a practical matter, the optimal classification domains in Refs.\ \cite{Patrone21,Luke22} often yield a partition that satisfies this property; thus this assumption is reasonable.  We interpret the existence of a partition satisfying (P3) to be a {\it minimum necessary condition} (which must be checked!) for a diagnostic assay to be useful, since it implies that the majority of samples will be classified correctly under any circumstances.  

The radius $\rhom$ also has the desirable property that $\rhom \to 0$ as $\P\to \I$.  This limit is consistent with the concept of an ideal (``gold-standard'') diagnostic for which each element of $U^\star$ corresponds unambiguously to a unique class \cite{Ideal}.  Under this scenario, the optimal solutions to Questions (I) and (II) also coincide, and as will become clear in Sec.\ \ref{sec:bounds}, the diagnostic assay attains its minimum uncertainty bounds.  For now, we observe that the limit $\P \to \I$ implies that the supports of the $p_j(\br)$ at most overlap on a set of measure zero, in which case there exists a perfect (no-error) classification partition.

In light of (P3) to (P6), note also that the Gershgorin circles associated with the columns of $\I-\P$ are centered at $1-P_{k,k}$ and have a radius $1-P_{k,k}$.  Thus, they are all tangent at the origin, and we conclude that the largest such circle contains them all.  Moreover, the identity matrix in the difference $\I-\P$ yields a unit-shift in the eigenvalues of $\P$; thus we likewise conclude that Gershgorin circles of $\P$ are all contained in the largest and are tangent at the point $(1,0)$ in the complex plane.  See Fig.\ \ref{fig:gershgorins}.

\section{Bounds of Classification Error and Uncertainty in Prevalence}
\label{sec:bounds}

Given that the elements of $\bt$ are non-negative and sum to one, P3 implies that the classification error $\mathcal E$ is bounded by the matrix 1-norm 
\begin{align}
\mathcal E \le \frac{1}{2}|| \I-\P ||_1 = \frac{1}{2}\max_k \sum_{j} |[\I-\P]_{j,k}| = \rhom. \label{eq:ebound}
\end{align}
See Ref.\ \cite{confusion} for discussion and related bounds.  Moreover, the estimate given by inequality \eqref{eq:ebound} is sharp. To see this, let $\kappa$ be the value of $k$ associated with $\rhom$, and set $q_\kappa=1$, $q_j=0$ for $j\ne \kappa$.

For the purposes of estimating prevalence, assume a partition $U=\{D_j\}$ satisfying (P3) and a collection of measurements $\mathcal S=\{\br_i\}$ whose underlying samples are drawn from classes $C_j$ according to $\bt$.  Recall that $\mathcal S$ can have repeated elements. Construct a vector $\hatQ$ having elements
\begin{align}
\hat {\mathscr q}_j = \frac{1}{s}\sum_{i=1}^{s} \mathbb{I}(\br_i \in D_j)
\end{align}
where $\mathbb{I}$ is the indicator function and $s$ is the cardinality of $\mathcal S$ (treating repeated elements as unique).  By invertibility of $\P$, we construct the estimator
\begin{align}
\hatq = \P^{-1}\hatQ.
\end{align}
Clearly the expectation $\E[\hatq]=\bt$, since $\hatQ$ is a Monte Carlo estimate of $\bQ$ \cite{montecarlo}.  We first seek to bound the variance
\begin{align}
\sigma^2(\I)= \E[(\hatq - \bt)^{\T}(\hatq- \bt)] = \E[(\P^{-1}\hatQ - \bt)^{\T}(\P^{-1}\hatQ - \bt)],
\end{align}
in terms of $\rhom$.

\begin{lemma}[Variance Bound]\label{lem:var}
Assume a partition $U$ such that the $c \times c$ matrix $\P$ satisfies property (P3) for $c\ge 2$.  Let $\mathcal S$ denote a collection of points as defined previously, and let $s$ be the number of elements in $\mathcal S$, with repeated elements counting as distinct. Then $\sigma^2$ satisfies the inequality
\begin{align}
\sigma^2=\sigma^2(\I) \le \frac{2c\rhom - \frac{c^2}{c-1}\rhom^2}{s(1-2\rhom)^2} + \sum_{j=1}^c\frac{q_j(1-q_j)}{s}.  \label{eq:varinequality}
\end{align}
\end{lemma}
\begin{remark}
We drop the dependence of $\sigma^2$ on $\I$ when clear from context.  
\end{remark}

\begin{proof} Decompose the difference $\hatq - \bt$ via conditional expectation as follows.  Let $\tq$ denote the random-vector with elements $\tilde q_j=m_j/s$, where $m_j$ is a random number of elements drawn from class $C_j$ according to a multinomial distribution with density $\bt$ and $s$ total events.  Clearly $\E[\tq]=\bt$.  Next express the difference
\begin{align}
\P^{-1}\hatQ - \bt = \P^{-1}(\hatQ - \P\tq) + (\tq - \bt).  \label{eq:decomp}
\end{align}
Observe that $\tq$ is the random vector associated with the fraction of samples in each class, whereas $\hatQ$ is the random vector associated with the fraction of samples in each domain of $U$.  
Taking the conditional expectation of Eq.\ \eqref{eq:decomp} yields
\begin{align}
\E[\P^{-1}\hatQ - \bt|\tq] = \tq - \bt. 
\end{align}
By the law of total variance (or conditional variance) \cite{Totvar}, we therefore recognize that
\begin{align}
\sigma^2 = \E\left[\E[(\P^{-1}\hatQ - \tq)^\T(\P^{-1}\hatQ - \tq)|\tq] \right] + \E[(\tq -\bt)^\T(\tq -\bt)]. \label{eq:totalvar}
\end{align} 
Known properties of the multinomial distribution yield \cite{multi}
\begin{align}
\E[(\tq -\bt)^\T(\tq -\bt)] = \sum_j \frac{q_j(1-q_j)}{s},
\end{align}
which is the second term on the right-hand side (RHS) of Eq.\ \eqref{eq:varinequality}.

Next consider that the vector 2-norm induces the corresponding matrix norm, yielding
\begin{align}
(\P^{-1}\hatQ - \tq)^\T(\P^{-1}\hatQ - \tq)=||\P^{-1}(\hatQ - \P\tq)||_2^2 \le ||\P^{-1}||_2^2 \,\, ||\hatQ - \P\tq||_2^2, \label{eq:conddecomp}
\end{align}
which implies that
\begin{align}
\E\left(||\P^{-1}(\hatQ - \P\tq)||_2^2 \big | \tq \right) \le ||\P^{-1}||_2^2 \,\, \E\left(||\hatQ - \P\tq||_2^2 \big | \tq\right) 
\end{align}
Writing the last term explicitly, one finds
\begin{align}
\E\left(||\hatQ - \P\tq||_2^2 \big | \tq\right) = \sum_{i,j} \frac{P_{i,j}(1-P_{i,j})\tilde q_j}{s}.
\end{align} 
Let $\rho_j = 1-P_{j,j}$.  One finds
 \begin{align}
 \sum_{i,j} \frac{P_{i,j}(1-P_{i,j})\tilde q_j}{s}& \le \max_{j}\left\{\sum_i P_{i,j}(1-P_{i,j}) \right\} \sum_k \frac{\tilde q_k}{s} = \frac{1}{s}\max_{j}\left\{\sum_i P_{i,j}(1-P_{i,j}) \right\} \nonumber \\
& = \frac{1}{s}\max_{j} \left\{\rho_j(1-\rho_j) + \sum_{i;i\ne j} P_{i,j} - P_{i,j}^2 \right\}  \nonumber \\
& = \frac{1}{s}\max_{j} \left\{\rho_j(2-\rho_j) - \sum_{i;i\ne j}   P_{i,j}^2 \right\} \nonumber \\
& \le \frac{1}{s} \max_j \left\{\rho_j(2-\rho_j) - \frac{1}{c-1} \left[\sum_{i:i\ne j}P_{i,j} \right]^2 \right\} \nonumber \\
&=  \frac{1}{s} \max_j \left\{\rho_j(2-\rho_j) - \frac{\rho_j^2}{c-1}  \right\} = \frac{2\rhom - \frac{c}{c-1}\rhom^2}{s}. \label{eq:long} \end{align}
The second and third lines come from property P2.  The fourth line is an application of the Cauchy-Schwarz inequality \cite{Krey}, and the fifth line leverages the pair of observations that $\rho_j \le \rhom < 0.5$ (properties P3 and P5; see also Fig.\ \ref{fig:gershgorins}) and $2\rho - \rho^2 - \frac{\rho^2}{c-1}$ is a monotone increasing function of $\rho$ for $0 < \rho < 0.5$ and $c\ge 2$.
Further recognize that,
\begin{align}
||\P^{-1}||_2^2 = \left[ \min_{\bx,||\bx||_2 = 1}|| \P \bx ||_2^2 \right]^{-1} \le c \left [  \min_{\bx,||\bx||_1 = 1}|| \P \bx ||_1^2 \right]^{-1} \le \frac{c}{(1-2\rhom)^2}, \label{eq:l2bound}
\end{align} 
where the last inequality arises by analyzing $\P$ element-wise.

Combining these equations one finds that
\begin{align}
\E\left(||\P^{-1}(\hatQ - \P\tq)||_2 \big | \tq \right) \le \frac{2c\rhom - \frac{c^2}{c-1}\rhom^2}{s(1-2\rhom)^2}.
\end{align} 
Since the conditional expectation is bounded by a deterministic quantity, we therefore arrive at inequality \eqref{eq:varinequality} via Eq.\ \eqref{eq:totalvar}. \qed

\end{proof}

\begin{corollary} The weighted variance $\sigma^2(A)$ is bounded by
\begin{align}
\sigma^2(\A) \le ||\A||_2^2 \sigma^2(\I),
\end{align}
where  $||\A||_2$ is the induced matrix 2-norm.
\end{corollary}

\begin{remark} Recall that $\P$ is both a push-forward operator \cite{Tao} and a stochastic matrix [in the sense of (P1)].  Its deviation from the identity characterizes the extent to which $\P$ ``redistributes'' probability mass when using $\bQ$ in lieu of $\bt$.  The $\rhom$ quantifies the maximum extent to which this mixing occurs.  This observation also explains why inequality \eqref{eq:varinequality} contains two terms, only one of which depends on $\rhom$.  The second term on the RHS quantifies the {\it inherent variation} due to the finite number of samples per class, which we can never eliminate except in the limit $s\to \infty$.  The first term on the RHS of \eqref{eq:varinequality} quantifies the excess uncertainty due to mixing of the probability density $\bt$.  This also clarifies the role of conditioning in Eq.\ \eqref{eq:conddecomp}, i.e.\ to separate these sources of uncertainty.  
\end{remark}
 
In practice, we can often make additional assumptions on the structure of $\P$ so as to sharpen inequality \eqref{eq:l2bound}.  The following corollary considers such a situation.
 
\begin{corollary}
Assume that $\P$ is symmetric.  Then we obtain the tighter bound
\begin{align}
\sigma^2 \le \frac{2\rhom - \frac{c}{c-1}\rhom^2}{s(1-2\rhom)^2} + \sum_j\frac{q_j(1-q_j)}{s}.  \label{eq:tightinequality}
\end{align} 
\end{corollary}
 
\begin{remark} Inequality \eqref{eq:tightinequality} differs from \eqref{eq:varinequality} in that the former removes a factor of dimension $c$ from the bound on variance.  See Sec.\ \ref{sec:minrho} for applications of this result.  
\end{remark}

\begin{proof}  Proceed exactly as in Lemma \ref{lem:var} up through inequalities \eqref{eq:long}.  In place of inequality \eqref{eq:l2bound}, recognize that
\begin{align}
||\P^{-1}||_2^2 = \frac{1}{\mathscr{s}_{\min}^2} = \frac{1}{|\lambda_{\min}|^2} \le \frac{1}{(1-2\rhom)^2},
\end{align}
where $\mathscr{s}_{\min}$ is the smallest singular value of $\P$, and $\lambda_{\min}$ is the eigenvalue with the smallest magnitude (see P5).  \qed 
\end{proof} 
 
\begin{remark}
In light of Lemma \ref{lem:var}, it is useful to define the following auxiliary quantities
\begin{subequations}
\begin{align}
&\varepsilon_{\rhom}=\frac{2c\rhom-\frac{c^2}{c-1}\rhom^2}{s(1-2\rhom)^2}, &  &\epsilon_{\rhom} = \frac{\varepsilon_{\rhom}}{c}, \\
&\varepsilon_{\sigma} = \varepsilon_{\rhom} + \sum_{j=1}^c \frac{q_j(1-q_j)}{s}, &  &\epsilon_{\sigma} = \epsilon_{\rhom} + \sum_{j=1}^c \frac{q_j(1-q_j)}{s},
\end{align}
\end{subequations}
where terms denoted by $\varepsilon$ and $\epsilon$ correspond to inequalities \eqref{eq:varinequality} and \eqref{eq:tightinequality}, respectively.  Note that $\varepsilon_{\rhom}$ and $\epsilon_{\rhom}$ are the excess uncertainty in prevalence due to overlap of the conditional probability densities. \label{rem:useful}
\end{remark} 
 
 
\section{Minimizing $\rhom$} 
\label{sec:minrho} 
 
When $\bt$ has more than two dimensions, it is not clear that there exist analytical formulas for the optimal partition $U^\star$ that minimizes $\rhom$.   However, the two-dimensional case is both extremely common in diagnostics (and thus useful) and analytically tractable; see, e.g.\ \cite{EUA}.  Therefore, consider the following problem:
\begin{align}
U^\star = {\rm \arg}\!\min_{U} [\rhom(U)]. \label{eq:opt_partition}
\end{align} 
 
The solution to Eq.\ \eqref{eq:opt_partition} can be constructed via a water-leveling algorithm that bears resemblance to the both the bathtub principle in measure theory \cite{bathtub} and methods for optimal classification \cite{Patrone22}.  The following definition establishes the latter connection.
\begin{definition}[Optimal Binary Classification Domains]
Let $p_1(r)$ and $p_2(r)$ be the conditional PDFs associated with classes $C_1$ and $C_2$.  Let $0 \le t < \infty$.  Then the sets 
\begin{subequations}
\begin{align}
\mathcal D_1(t) &= \{\br: p_1(\br)  > t p_2(\br)\} \cup \mathcal B_1, \label{eq:D1}\\
\mathcal D_2(t) &= \{\br: p_1(\br)  < t p_2(\br)\} \cup \mathcal B_2, \label{eq:D2}\\
\mathcal D_b(t) &= \{\br: p_1(\br) = t p_2(\br)\} =  \mathcal B_1 \cup  \mathcal B_2, \label{eq:Db}
\end{align}
\end{subequations}
are the {\bf optimal binary classification domains} associated with prevalence ${q=q_1=(1+t)^{-1}}$, where $\mathcal  B_1$ and $\mathcal  B_2$ are arbitrary partitions of the boundary set $\mathcal D_b$.  
\end{definition}
\begin{remark}
It is straightforward to show that Eqs.\ \eqref{eq:D1} and \eqref{eq:D2} minimize the error \eqref{eq:error} in the binary classification setting.  See, e.g.\ Refs.\ \cite{Patrone21,RW}.
\end{remark}
Construction of $U^\star$ also makes use of the corresponding measures
\begin{align}
\mu_j(t) &= \int_{\mathcal D_j(t)}p_j(\br) \d \br, & \mu_{b,j}(t)= \int_{\mathcal B_j(t)} p_j(\br) \d \br,
\end{align}
for $j=1,2$.  The optimal partition $U^\star=\{D_1^\star,D_2^\star\}$ is then given by the point at which the difference $\Delta(t) = \mu_1(t) - \mu_2(t)$ crosses zero.  More precisely:
\begin{lemma}[Water-Leveling Construction]\label{lem:water}
Assume that the function $\Delta(t)$ attains the value $\Delta(t^\star)=0$ for some $t^\star$ on $0 < t < \infty$ and $\mathcal D_b(t^\star)$ has zero measure with respect to both $p_1(\br)$ and $p_2(\br)$.  Then the partition  
\begin{subequations}
\begin{align}
D_1^\star &= \mathcal D_1(t^\star) \cup \mathcal B_1(t^\star) \label{eq:D1star} \\
D_2^\star &= \mathcal D_2(t^\star) \cup \mathcal B_2(t^\star) \label{eq:D2star}\\
\mathcal D_b(t^\star) &= \mathcal B_1(t^\star) \cup \mathcal B_2(t^\star)     \label{eq:Dbstar}
\end{align}
\end{subequations}
minimizes $\rhom$, where $\mathcal B_1(t^\star)$ and $\mathcal B_2(t^\star)$ are an arbitrary partition of the set $\mathcal D_b(t^\star)$.  Alternatively, if (i) $\mathcal D_b(t^\star)$ has positive measure with respect to any distribution or (ii) $\Delta(t)$ is discontinuous at $t^\star$ and $\pm \Delta(t^\star \pm \epsilon) < 0$ for any positive $\epsilon$, Eqs.\ \eqref{eq:D1star}--\eqref{eq:Dbstar} minimize $\rhom$ subject to the additional constraint on $\mathcal B_1(t^\star)$ and $\mathcal B_2(t^\star)$ that 
\begin{equation}
\mu_1(t^\star) + \mu_{b,1}(t^\star)-\mu_2(t^\star) - \mu_{b,2}(t^\star) = 0.
\end{equation}
\end{lemma}

\begin{proof} Observe that $\Delta(t)$ is a monotone decreasing function of $t$.

Consider the case in which $\Delta(t^\star)=0$, and temporarily assume that the set $\mathcal D_b(t^\star)$ has zero measure.  Note that $\I-\P$ is symmetric (since $\Delta(t^\star)=0$) and has the form
\begin{align}
\I-\P = \begin{bmatrix}
1-\mu_1(t^\star) & \mu_2(t^\star)-1 \\
\mu_1(t^\star) -1 & 1-\mu_2(t^\star)
\end{bmatrix}. \label{eq:symmat}
\end{align}
Choose an arbitrary subset $\delta_1 \in \mathcal D_1(t^\star)$ with positive measure with respect to $p_1(\br)$.  Clearly we cannot remove such a $\delta_1$  from $D_1(t^\star)$ and add it to $D_2(t^\star)$, since this will increase the sum of absolute values in the first column, and thus $\rhom$.  Neither can we exchange points between domains.  To see this, define a second set $\delta_2\in \mathcal D_2(t^\star)$ having positive measure with respect to $p_2(\br)$, and consider the sets $\bar D_1 = (D_1(t^\star) / \delta_1) \cup \delta_2$ and $\bar D_2 = (D_2(t^\star) / \delta_2) \cup \delta_1$ with corresponding measures $\bar \mu_1$ and $\bar \mu_2$.  Observe that the differences in measures is given by
\begin{subequations}
\begin{align}
\bar \mu_1 -\mu_1(t^\star) =  - \int_{\delta_1} p_1(\br) \d \br + \int_{\delta_2} p_1(\br) \d \br \label{eq:diff1} \\ 
\bar \mu_2 -\mu_2(t^\star) =  - \int_{\delta_2} p_2(\br) \d \br + \int_{\delta_1} p_2(\br) \d \br. \label{eq:diff2}
\end{align}
\end{subequations}
Using the definitions of $\mathcal D_1(t^\star)$ and $\mathcal D_2(t^\star)$, \eqref{eq:diff1} and \eqref{eq:diff2} imply
\begin{align}
\bar \mu_1 -\mu_1(t^\star) < -t^\star [\bar \mu_2 -\mu_2(t^\star)]. \label{eq:t}
\end{align}
Recognize that $t^\star > 0$.  Thus, there are three possibilities: the RHS of Eq.\ \eqref{eq:t} is positive, negative, or zero.  In all three cases, $\bar \mu_j < \mu_j(t^\star)$ for either $j=1$ or $j=2$.  We thereby conclude that the minimum value of $\rhom$ is given by the zero of $\Delta(t)$, which is unique by monotonicity and the assumption that $\mathcal D_b(t^\star)$ has zero measure.  Note that in this case, the sets given by Eqs.\ \eqref{eq:D1} and \eqref{eq:D2} automatically partition $\Gamma$ up to a set of measure zero, whose points we may assign to either $D_1$ or $D_2$.

When $\Delta(t^\star)=0$ but the set $\mathcal D_b(t^\star)$ does not have measure zero, select $\mathcal B_1$ and $\mathcal B_2$ such that $\mu_{b,1}(t^\star)-\mu_{b,2}(t^\star)=0$.  We then apply the steps leading to inequality \eqref{eq:t}.  This same argument applies to the case that $\Delta(t)$ is discontinuous at $t^\star$, with sole modification that the difference $\mu_{b,1}(t^\star)-\mu_{b,2}(t^\star)$ be equal to the jump discontinuity in $\Delta t$ at $t^\star$.  \qed

\end{proof}

\begin{figure}\begin{center}
\includegraphics[width=12cm]{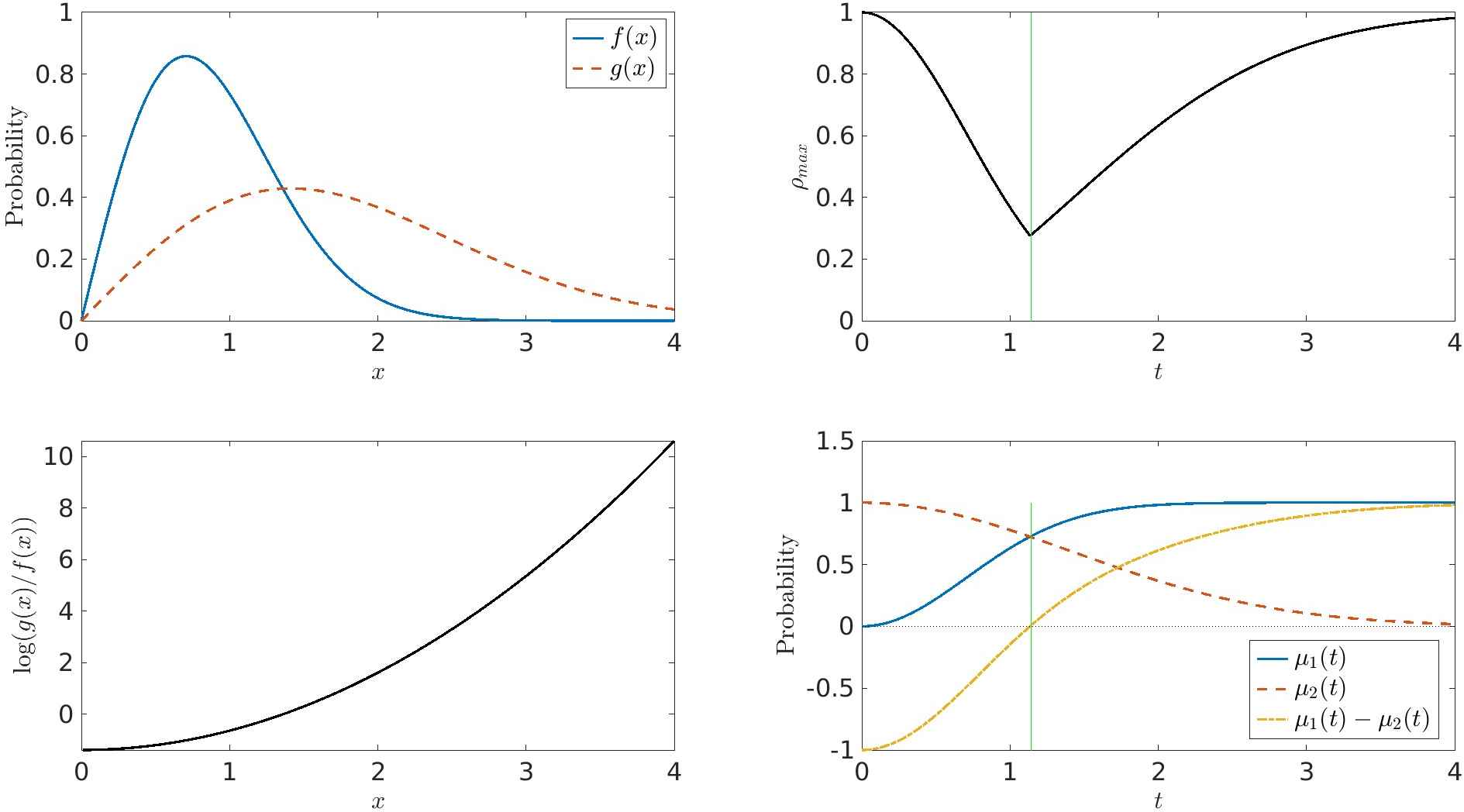}\end{center}\caption{Example of the analysis yielding the optimal $\rhom$ for a pair of distributions.  The PDFs $f(x)$ and $g(x)$ are Weibull distributions with scale parameters of 1 and 2, respectively, and both having a shape parameter of 2.  The top-left plot shows the underlying distributions.  The bottom-left plot shows the logarithm of the ratio of PDFs.  Clearly, this ratio is a monotone increasing function of $x$.  This implies that the sets $D_1$ and $D_2$ are defined by the inequalities $x \le t$ and $x \ge t$.  The top-right plots shows $\rhom(t)$ as a function of $t$, whereas the bottom plot shows the measures $\mu_1(t)$ and $\mu_2(t)$ as a function of $t$.  Note that the value of $t^\star$ given by the intersection of $\mu_1(t)$ and $\mu(t)$ is the point for which $\rhom(t)$ is minimized.}\label{fig:wblexamps}
\end{figure}

\medskip

Figure \ref{fig:wblexamps} illustrates the solution given by Lemma \ref{lem:water}.  In particular, we consider two Weibull PDFs on the domain $x\in [0,\infty)$ defined by \cite{multi}
\begin{align}
f(x)&= 2xe^{-x^2} \\ 
g(x)&= (x/2) e^{-(x/2)^2}.
\end{align} 
The four panels of Fig.\ \ref{fig:wblexamps} show: the PDFs (top left); the logarithm of the ratios of PDFs (bottom left); $\rhom$ as a function of the $t$ appearing in the definitions of $D_1(t)$ and $D_2(t)$ (top right); and the measures $\mu_1(t)$ and $\mu_2(t)$, as well as their differences (bottom right).  Note that the minimum value of $\rhom(t)$ corresponds to the value of $t$ for which $\mu_1(t)-\mu_2(t)=0$, as indicated by the vertical green line.  
 
\begin{figure}\begin{center}
\includegraphics[width=13.5cm]{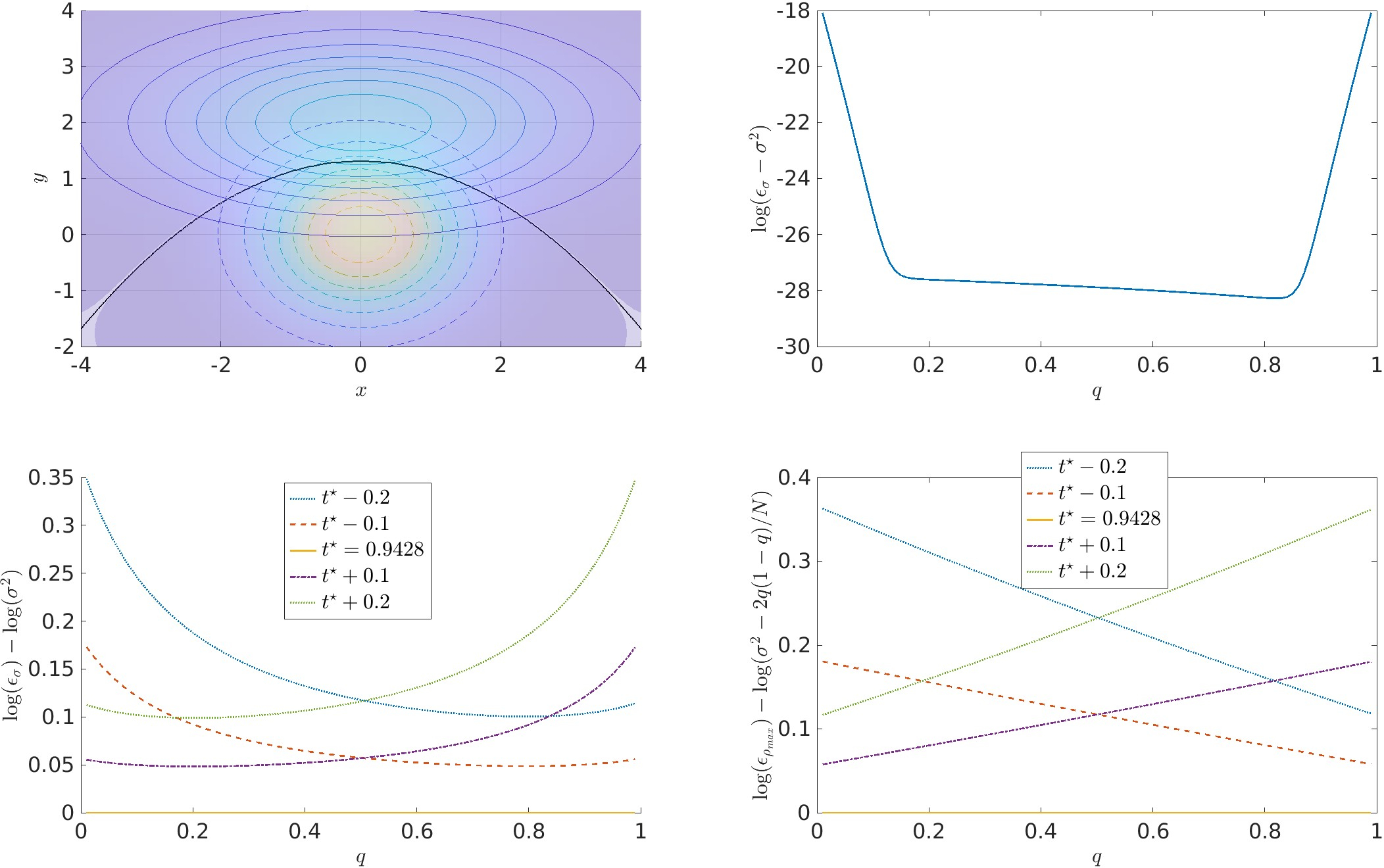}\end{center}\caption{{\it Top left:} Heat maps and contour plots for the PDFs given by Eqs.\ \eqref{eq:f1} and \eqref{eq:g1}.  The purple background corresponds to regions where the PDFs approach zero.  The solid parabola is the boundary associated with $t^\star$.  {\it Top right:}  The logarithm of the difference between the upper bound $\epsilon_\sigma$ given by the RHS of \eqref{eq:tightinequality} and $\sigma^2(\I)$.  {\it Bottom-left:} Logarithm of the ratio of $\epsilon_\sigma$ and $\sigma^2(\I)$.  {\it Bottom-right:} Logarithm of the ratio of $\epsilon_{\rhom}$ and $\sigma^2(\I) - 2q(1-q)/s$.  The latter difference is the excess uncertainty due to the imperfect diagnostic test.  For all plots, we take $s=100$.}\label{fig:2dpdfs}
\end{figure} 
 
Figure \ref{fig:2dpdfs} illustrates the implications of Lemma \ref{lem:water} in the context of the bounds given in Sec. \ref{sec:bounds}.  The top-left plot superimposes heat maps and contour plots for two PDFs given by
\begin{subequations}
\begin{align}
f(x,y) &= \frac{1}{2\pi}e^{-x^2/2 - y^2/2}, \label{eq:f1} \\
g(x,y) &= \frac{1}{4\pi}e^{-x^2/8 - (y-2)^2/2}. \label{eq:g1}
\end{align}
\end{subequations}
The parabola is the boundary associated with $t^\star = 0.9428$, which defines the minimum value of $\rhom$.  The top-right plot shows $\log(\sigma^2 - \epsilon_{\sigma})$, both as a function of prevalence; see Remark \ref{rem:useful}.  This plot shows that $\epsilon_\sigma$ is a tight bound to within the precision with which the computations were done.  The bottom-left plot shows $\log(\epsilon_{\sigma} / \sigma^2)$ for different choices of prevalence estimation domains, illustrating the optimality of the sets associated with $t^\star$.  The bottom-right plot shows the corresponding excess uncertainty relative to $\sigma^2$ for optimal and non-optimal values of $t^\star$.

\section{Discussion}
\label{sec:discussion}

\subsection{Challenges of Generalizing Lemma \ref{lem:water}}

An interesting observation of Lemma \ref{lem:water} is that the diagonal elements of $\P$ are equal when $\rhom$ is minimized.  It seems reasonable that this property should carry over to the multiclass setting and thereby generalize Eqs.\ \eqref{eq:D1star}--\eqref{eq:Dbstar}.  However, the situation is not straightforward.  A simple example illustrates the key challenges.

\begin{example}[Partially Disjoint Uniform Distributions]  Define
\begin{align}
u(x;y) = \begin{cases}
1 & y-1/2 \le x \le y+1/2 \\
0 & {\rm otherwise}
\end{cases}
\end{align}
which is the uniform distribution in $x$ on a unit interval centered at $y$.  Let
\begin{subequations}
\begin{align}
p_1(r)&=u(r;0.5) \label{eq:u1} \\
p_2(r)&=u(r;1.4) \\
p_3(r)&=u(r;2).\label{eq:u3}
\end{align}
\end{subequations}
See also Fig.\ \ref{fig:3pdfs}.  
\begin{figure}
\includegraphics[width=12cm]{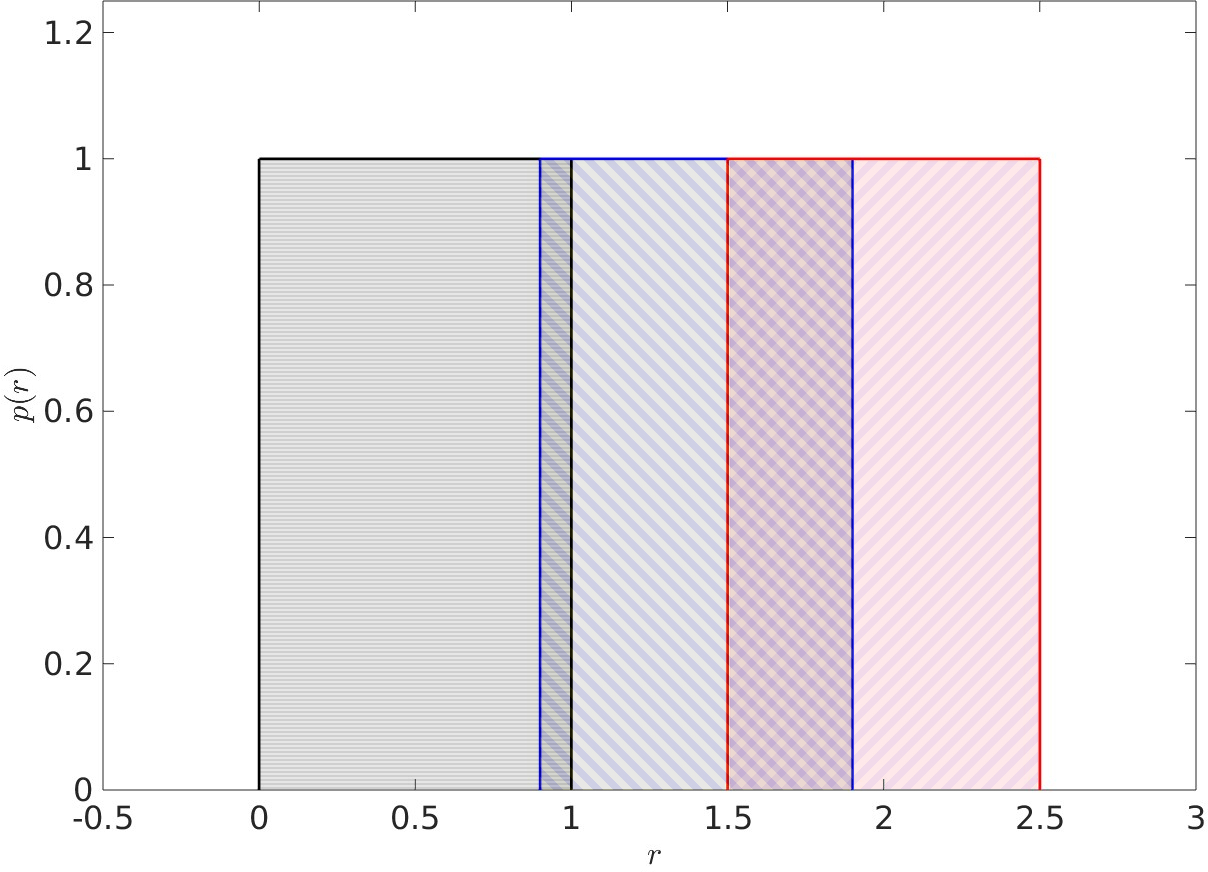}\caption{Three uniform distributions given by Eqs.\ \eqref{eq:u1}--\eqref{eq:u3}.  Note that there is significantly less overlap between $p_1(r)$ and $p_2(r)$ than between $p_2(r)$ and $p_3(r)$.}\label{fig:3pdfs}
\end{figure}
Note that for these distributions, a partition $U$ can be defined in terms of two points $x_{1,2}$ and $x_{2,3}$ that act as boundaries between classes $C_j$ and $C_{j+1}$, $j=1,2$.  It is straightforward to show that an optimal partition $U^\star$ is given by $x_{1,2}=0.9$ and $x_{2,3}=1.7$, which corresponds to the matrix 
\begin{align}
\P(U^\star) = \begin{bmatrix}
0.9 & 0 & 0 \\
0.1 & 0.8 & 0.2 \\
0 & 0.2 & 0.8
\end{bmatrix}, \label{eq:exampmat}
\end{align}
having $\rhom=0.2$.  We could also choose $x_{1,2}=0.8$, which would make all three diagonal entries identical without changing $\rhom$.  However, this would be a poor decision from a diagnostic standpoint, since only $p_1(r)$ has probability mass on the domain $r\in [0.8,0.9]$.  
\end{example}

This example illustrates several important properties that differentiate the multiclass setting from its binary counterpart.  As an analogy, first observe that we may interpret each class as the $j$th vertex or node in a graph, with an edge connecting vertices $j$ and $j'$ if $p_j(r)$ and $p_{j'}(r)$ have overlapping support.  Thus, connected nodes indicate pairs of PDFs between which we can redistribute probability mass (by changing $U$) to minimize $\rhom$.  In the binary problem, there is only one such edge, whereas in the multiclass problem, mass can be redistributed along several distinct paths.  Compounding this, the mass that can be transferred from class $C_j$ to $C_{j'}$ may not equal the mass that can be transferred from $C_{j'}$ to $C_{j''}$.  This implies that two optimal partitions $U^\star$ and $V^\star$ can differ by more than boundary sets defined by analogy to Eq.\ \eqref{eq:Db}.

Despite this complication, we can still deduce certain properties of the collection $\mathcal P$ of matrices $\P$ that minimize $\rhom$, and these properties may be useful for constructing useful optimization algorithms.  The following lemma is motivated by the structure of Eq.\ \eqref{eq:symmat}.
\begin{lemma}[Constant Diagonal Matrix]\label{lem3}
Let $\mathcal P$ be the set of $c\times c$ matrices $\P$ that minimize $\rhom$ for a given set of conditional PDFs $p_k(\br)$ for $\br\in \Gamma$, $k=1,2,...,c$.       Then $\mathcal P$ contains a matrix whose diagonal entries are equal.
\end{lemma}

\begin{proof}[By Construction]  Assume $\P \in \mathcal P$.  Let $K=\{k_1,k_2,...,k_m \}$ denote the columns for which $P_{k,k}$ attains its minimum.  If $m=c$, then the lemma is trivially true.  

Assume therefore that $m < c$.  Recall that by definition, $P_{j,k}$ is the probability mass from the $k$th distribution contained in domain $D_j$.  Observe that there must exist at least one $k'\in K$ such that
\begin{align}
{\rm supp}[p_{k'}(\br)] \bigcap \left[\bigcup_{j:j\notin K} D_j \right] = \emptyset.  \label{eq:noverlap}
\end{align}
That is, there must exist at least one $k'$ such that the support of $p_{k'}(\br)$ is disjoint from every $D_j$ when $j\notin K$.  To prove Eq.\ \eqref{eq:noverlap}, argue by contradiction.  Assuming the result is false implies that for each $k\in K$ there exists at least one $j_k \notin K$ such that $p_k(\br)$ has finite measure on $D_{j_k}$.  Let $\epsilon > 0$ be the magnitude of the difference between the smallest two {\it unique}, diagonal elements of $\P$.  By continuity of measure, for each $k\in K$, we may choose a domain $D_{k,j_k}\subset D_{j_k}$ such that: (i) the $D_{k,j_k}$ are disjoint; and (ii) $0 < \mu_k(D_{k,j_k}) < \epsilon/(2c)$ and $0 \le \mu_j(D_{k,j_k}) < \epsilon/(2c)$, where $\mu_k(D)$ is the measure of $D$ with respect to the $k$th distribution.  Define the new domains
\begin{align}
\tilde D_{k} &= D_k \bigcup D_{k,j_k} & k\in K \\
\tilde D_j &= D_j / \left[\bigcup_{k: k\in K} D_{k,j_k} \right], & j\notin K
\end{align}
and let $\tilde \P$ be the induced confusion matrix.  Clearly
\begin{align}
\min_n P_{n,n} < \min_n \tilde P_{n,n},
\end{align}
since we have at most increased the smallest diagonal element of $\P$ by $\epsilon/2$ and decreased the next-smallest diagonal element by $\epsilon/2$.
This implies that the $\rhom$ associated with $\I-\tilde \P$ is less than that corresponding to $\I-\P$,  contradicting the assumption $\P \in \mathcal P$.  Thus Eq.\ \eqref{eq:noverlap} is true.  

To prove the main result, let $k\in K$ correspond to any $p_k(r)$ for which Eq.\ \eqref{eq:noverlap} is true, since there must be at least one.  Let 
\begin{align}
\epsilon_j = P_{j,j} - \min_n P_{n,n}.
\end{align}
For each $j\notin K$, we may choose a subset $\hat D_j \subset D_j$ such that $\mu_j(\hat D_j) = \epsilon_j$ and $\mu_k(\hat D_j)=0$.  Clearly the partition $U'$ defined as
\begin{subequations}
\begin{align}
D_k' &= D_k \bigcup \left[ \bigcup_{j:j\ne K} \hat D_j \right] \\
D_j' &= D_j / \hat D_j, & j\notin K \\
D_{n}' &= D_n & n\in K, n\ne k
\end{align}
\end{subequations} 
induces a matrix $\P'$ for which the largest Gershgorin radius of $\I-\P'$ is $\rhom$.  However, by construction, we have removed exactly $P_{j,j}-P_{k,k}$ from each diagonal element of $\P$, so that they are all constant.  \qed

\end{proof}

\begin{remark} Lemma \ref{lem3} mirrors the example at the beginning of this section.  While we are not guaranteed that a $\P\in \mathcal P$ has a constant diagonal, we can always shift probability mass to enforce this criterion.  However, as the example also illustrates, doing so does not necessarily yield an optimal diagnostic interpretation of the assay.  Nonetheless, the existence of such a matrix resolves uniqueness and suggests that one route to minimizing $\rhom$ is to solve the optimization problem
\begin{align}
\rhom^\star = \min_{U} [1-P_{j,j}(D_j)]
\end{align}
for any $j$, subject to the constraints
\begin{subequations}
\begin{align}
P_{j,j}(D_j) &= P_{k,k}(D_k) \\
D_j \bigcap D_k &= \emptyset
\end{align} 
\end{subequations}
for all $k \ne j$.   Note that the resulting $\P$ is not guaranteed to be a Toeplitz matrix \cite{toeplitz}.
\end{remark}
 
In some cases, it is possible to strengthen Lemma \ref{lem3} given optimal classification domains.  In particular, assume $c$ classes and let $q_k \in [0,1]$, $k=1,2,...,c$ satisfy 
\begin{align}
\sum_{k=1}^c q_k = 1. \label{eq:qsum}
\end{align} 
\begin{definition}
The sets
\begin{align}
D_k^\star(\bt) = \{r:q_k p_k(\br) > q_jp_j(\br) \,\, \forall j: j\ne k\} \label{eq:optclass}
\end{align}
are  optimal classification domains associated with the generalized prevalence $\bt=(q_1,q_2,...,q_c)$ if the $D_k^\star(\bt)$ partition $\Gamma$ up to sets of measure zero in all of the $p_k(\br)$. \label{def:multi}
\end{definition}

It is straightforward to show that the sets defined by Eq.\ \eqref{eq:optclass} minimize the classification error given by Eq.\ \eqref{eq:error} when all $\br$ have zero measure and the sets partition $\Gamma$ \cite{Luke23}.  In fact, these sets mirror the structure of Eqs.\ \eqref{eq:D1} and \eqref{eq:D2}.  This observation motivates the following lemma.
 
\begin{lemma}\label{lem4}
Assume that there exists a discrete probability density $\bt=(q_1,q_2,...,q_c)$ for which the $c\times c$ confusion matrix $\P^\star$ induced by $ D_k^\star(\bt^\star)$ has equal diagonal entries.  That is, $P_{k,k}^\star = P_{j,j}^\star$ for all pairs $j$ and $k$.  Then the $\rhom$ associated with $\P^\star$ attains its minimum.  
\end{lemma} 
 
\begin{proof} Note first that for any overlap matrix $\P$,
\begin{align}
\sum_{k=1}^c q_k P_{k,k} \ge \min_k \{P_{k,k}\}  \sum_{k=1}^c q_k = \min_k \{P_{k,k}\} = 1-\rhom, \label{eq:qineq}
\end{align}
since $\bt$ is a probability measure.  Moreover, this inequality is sharp for $\P^\star$, since its diagonal entries are all the same.  That is,
\begin{align}
\sum_{k=1}^c q_k P_{k,k}^\star = P_{j,j}^\star = 1-\rhom^\star,
\end{align}
for any value of $j=1,2,...,c$.  

To show that $\rhom^\star = \min\{\rhom\}$, let  $U'=\{D_k'\}$ be any other partition of $\Gamma$, and let $\P'$ be the induced confusion matrix.   One finds that
\begin{align}
\sum_{k=1}^c q_k P_{k,k}' &= \sum_{k=1}^c \int_{D_k'} q_k p_k(\br) \d \br \nonumber \\
&= \sum_{k=1}^c \left[  \int_{D_k'\cap D_k^\star} q_k p_k(\br) \d \br + \sum_{j:j\ne k} \int_{D_k' \cap D_j^\star} q_k p_k(\br) \d \br\right ] \nonumber \\
& \le \sum_{k=1}^c \left[  \int_{D_k'\cap D_k^\star} q_k p_k(\br) \d \br + \sum_{j:j\ne k} \int_{D_k' \cap D_j^\star} q_j p_j(\br) \d \br\right ] \nonumber \\
&= \sum_{k=1}^c \sum_{j=1}^c \int_{D_k' \cap D_j^\star} q_j p_j(\br) \d \br = \sum_{j=1}^c \int_{D_j^\star} q_j p_j(\br) \d \br \nonumber \\
&= \sum_{k=1}^c q_k P_{k,k}^\star, \label{eq:longineq}
\end{align}
since the $D_k'$ and $D_k^\star$ are both partitions of $\Gamma$.  Note that: (i) the third line of \eqref{eq:longineq} is a consequence of the structure of $D_k^\star$; and (ii) inequality \eqref{eq:longineq} amounts to a proof of the optimality of the sets in Def. \ref{def:multi}.  Combining this result with Eq.\ \eqref{eq:qineq}, we find
\begin{align}
1-\rhom(U') \le \sum_{k=1}^c q_k P_{k,k}' \le \sum_{k=1}^c q_k P_{k,k}^\star = 1-\rhom^\star.
\end{align}
Thus, the result holds, since $U'$ was an arbitrary partition different from $U^\star$.  \qed

\end{proof}

Lemma \ref{lem4} suggests that an alternate route for minimizing $\rhom$ in a multiclass setting is to find the prevalence $\bt$ such that the optimal classification domains have the same probability mass.  In practice, this is likely reasonable; the PDFs $p_k(\br)$ underlying diagnostic tests are typically unimodal, smooth, and monotone decreasing away from their maxima, which suggests that the $P_{k,k}^\star$ have the needed continuity of measure.   However, this assumption may not be sufficient to guarantee the hypotheses of Lemma \ref{lem4} are satisfied.  In general, we do not know what conditions imply the hypotheses of Lemma \ref{lem4}, although we speculate that the Poincare-Miranda theorem may suffice \cite{PM}.  Such questions are left for future work.

\subsection{Effects of Noise}
\label{subsec:noise}

In many applications, we do not have access to the random variable $\br(\omega)$ directly, but rather only a noisy approximation $\sr$.  The goal of this section is to understand the impact of noise on $\rhom$.  [See Ref.\ \cite{Noisy_labels} for an analysis of the distinct problem wherein the training classes are noisy.]

To make this more precise, assume that $\sr$ suffers from additive noise of the form
\begin{align}
\sr=\br+\eta, \label{eq:noisyrv}
\end{align}
where $\eta$ is an uncorrelated random variable.  In practice, $\eta$ arises from instrument noise associated with the measurement process that maps $\omega$ onto $\br(\omega)$.  Thus, to make specification of $\eta$ fully precise, it is necessary to further expand the notion of the sample space so that each $\omega$ corresponds to the triple of a random individual's class, antibody level(s), and the instrument noise.  We assume that the corresponding $\sigma$-algebra and probability measures always exist.

Equation \eqref{eq:noisyrv} implies that the matrix elements $P_{j,k}$ must be re-interpreted in terms of the usual convolution formula as
\begin{subequations}
\begin{align}
\tilde P_{j,k}&=\int_{D_j} \d \br \int_{\Gamma} \d \bx\,\, \mathcal N(\br-\bx)p_{k}(\bx) = \int_{D_j} \d \br\,\, \tilde p_k(\br) \label{eq:noisymelement} \\
\tilde p_k(\br) &= \int_{\Gamma} \d \bx\,\, \mathcal N(\br-\bx)p_{k}(\bx) \label{eq:noisypdf}
\end{align}
\end{subequations}
where $\mathcal N(\bx)$ is the PDF associated with $\eta$.  For concreteness and to avoid inconsistencies in the meaning of the domain $\Gamma$, assume that $\Gamma = \mathbb R^n$ and $\mathcal N$ is the PDF of a multi-variate normal random variable with mean-zero and positive-definite covariance matrix $\Psi$.  Without loss of generality, we may express $\Psi$  as $\Psi = \varsigma^2 \Phi$ for $\Phi$ a constant, positive-definite matrix whose largest eigenvalue is $1$.  Following the notation of Lemma \ref{lem:water}, let $\rhom^\star(\varsigma^2)$ be the minimum value of $\rhom$ associated with the matrix elements $\tilde P_{j,k}$ defined in terms of the $\tilde p_k(\br)$, and let the $\rhom^\star(0)$ be the minimum associated with the original PDFs $p_k(\br)$.  We seek to answer the following question: does $\rhom^\star(\varsigma^2)$ increase with $\varsigma^2$?  

\begin{lemma}[Noise Decreases Optimal Performance]\label{lem:noise}
Let $\tilde P_{j,k}$ be defined as in Eq.\ \eqref{eq:noisymelement}, let $c=2$ (binary setting), and assume that $\mathcal N$ is the PDF of a mean-zero, normal random variable with positive-definite covariance matrix $\Psi=\varsigma^2 \Phi$, where the largest eigenvalue of $\Phi$ is one.  Assume also that $q_1 \in (0,1)$. Then $\rhom^\star(\varsigma^2)$ is a monotone increasing function of $\varsigma^2$.  
\end{lemma}

\begin{proof}
By the infinite divisibility of the normal distribution \cite{Divisible1,Divisible2}, it suffices to show that $\rhom^\star(\varsigma^2) \ge \rhom^\star(0)$ for any $\varsigma^2$.  Let $P_{j,k}$ be the corresponding matrix elements when $\varsigma=0$.  Note that the optimal partitions $\tilde U$ and $U$ defining $\rhom^\star(\varsigma^2)$ and $\rhom^\star(0)$ are not necessarily the same.  Moreover, by Lemma \ref{lem:water}, the partitions $\tilde U$ and $U$ can be defined in terms of classification domains $\tilde D_j^\star =\tilde{\mathcal D}_j(\tilde t^\star)\cup \tilde{\mathcal B}_j(\tilde t^\star)$ and $D_j^\star = \mathcal D_j(t^\star)\cup \mathcal B_j(t^\star)$ for some parameters $\tilde t^\star$ and $t^\star$, such that, for example
\begin{align}
D_1^\star = \{\br: p_1(\br) > t^\star p_2(\br)\} \cup \mathcal B_1(t^\star). \label{eq:recallprevs}
\end{align}

Recall Lemma \ref{lem:water} and consider the difference
\begin{align}
\rhom^\star(\varsigma^2) - \rhom^\star(0) &= \tilde P_{1,2} - P_{1,2}. 
\end{align}
Because the partitions $\tilde U$ and $U$ cover $\Gamma$, one finds
\begin{align}
\tilde P_{1,2} - P_{1,2} &= \dint{\tilde D_1^\star}{\tilde D_1^\star}\np{2} + \dint{\tilde D_1^\star}{\tilde D_2^\star}\np{2} \nonumber \\
& \qquad - \dint{\tilde D_1^\star}{D_1^\star}\np{2} - \dint{\tilde D_2^\star}{D_1^\star}\np{2} \nonumber \\
&= \dint{\tilde D_1^\star}{\tilde D_1^\star / D_1^\star}\np{2} - \dint{\tilde D_1^\star}{D_1^\star / \tilde D_1^\star}\np{2} \nonumber \\
& \qquad + \dint{\tilde D_1^\star}{\tilde D_2^\star}\np{2} - \dint{\tilde D_2^\star}{D_1^\star}\np{2}, \label{eq:longints}
\end{align}
where the last line arises by considering the difference of the first and third terms in the top equality.
Next recall the identity
\begin{align}
\tilde D_2^\star = [\tilde D_2^\star \cap D_1^\star] \cup [\tilde D_2^\star / D_1^\star]
\end{align}
and observe that 
\begin{align}
\tilde D_2^\star \cap D_1^\star = D_1^\star / \tilde D_1^\star.  
\end{align}
Combining these last three equalities implies that
\begin{align}
\tilde P_{1,2} - P_{1,2} &= \dint{\tilde D_1^\star}{\tilde D_1^\star / D_1^\star}\np{2} + \dint{\tilde D_1^\star}{\tilde D_2^\star / D_1^\star}\np{2} \nonumber \\ 
&\qquad - \dint{\tilde D_2^\star}{D_1^\star}\np{2} \nonumber \\
&= \dint{\tilde D_1^\star}{D_2^\star}\np{2} - \dint{\tilde D_2^\star}{D_1^\star}\np{2} \label{eq:corn1}
\end{align}
By symmetry of the matrix elements (see Lemma \ref{lem:water}), we also find that
\begin{align}
\tilde P_{1,2} - P_{1,2} &= \tilde P_{2,1} - P_{2,1} \nonumber \\
&=\dint{\tilde D_2^\star}{D_1^\star}\np{1} - \dint{\tilde D_1^\star}{D_2^\star}\np{1}.\label{eq:corn2}
\end{align}

Create a linear combination of Eqs.\ \eqref{eq:corn1} and \eqref{eq:corn2} to find
\begin{align}
(1+a)(\tilde P_{1,2} - P_{1,2}) &= a\dint{\tilde D_1^\star}{D_2^\star}\np{2} - a\dint{\tilde D_2^\star}{D_1^\star}\np{2} \nonumber \\
&\quad + \dint{\tilde D_2^\star}{D_1^\star}\np{1} - \dint{\tilde D_1^\star}{D_2^\star}\np{1} \nonumber
\end{align}
for an arbitrary constant $a > -1$.  Recalling Eq.\ \eqref{eq:recallprevs}, the set $D_2^\star$ is one for which
\begin{align}
p_1(\br) \le t^\star p_2(\br).
\end{align}
Thus, we find that
\begin{align}
\tilde P_{1,2} - P_{1,2} &\ge \frac{a-t^\star}{1+a} \dint{\tilde D_1^\star}{D_2^\star}\np{2} \nonumber \\
&\qquad + \frac{1-a/t^\star}{1+a}\dint{\tilde D_2^\star}{D_1^\star}\np{1}.
\end{align}
Since $a$ may be chosen to be any positive number, and $t^\star > 0$ by definition, set $a=t^\star$.  This implies that $\tilde P_{1,2} - P_{1,2} \ge 0$, and thus  $\rhom^\star(\varsigma^2) - \rhom^\star(0) \ge 0$.  \qed
\end{proof}

\begin{remark}
Lemmas \ref{lem:water} and \ref{lem:noise} imply that in the binary case, adding Gaussian noise to the random variable $\br$ never decreases $\rhom^\star$.  It is important to note that this result may not be true for a non-optimal $\rhom$, as the next example illustrates.
\end{remark}

\begin{example}[Noise Increasing Accuracy]
Let $r\in \mathbb R$ and define
\begin{subequations}
\begin{align}
p_1(r) = \begin{cases}
1 & 0 \le r \le 0.6 \\
1 & 1 \le r \le 1.4 \\
0 & {\rm otherwise}
\end{cases}
\end{align}
and 
\begin{align}
p_2(r) = \begin{cases}
1 & 10 \le r \le 11 \\
0 & {\rm otherwise}
\end{cases}.
\end{align}
\end{subequations}
Choose $D_1=(-\infty,1)$ and $D_2=[1,\infty)$.  Clearly $\rhom(0)=0.4$ in this case.  

Select $0 < \varsigma \ll 1$.  Straightforward calculations show that $\rhom(\varsigma^2) < \rhom(0)$, with the difference being approximated by 
\begin{align}
\rhom(0) - \rhom(\varsigma^2) = \frac{1}{\sqrt{2\pi}\varsigma} \int_{-\infty}^{1} \!\!\!\!\!\!\!\! \d r \int_{1}^{1.4}\d z\,\, e^{-\frac{(r-z)^2}{2\varsigma^2}} + \mathcal R, \label{eq:errorapprox}
\end{align}
where the remainder $\mathcal R$ is an exponentially small error term.  
\end{example}

\begin{remark}
Equation \eqref{eq:errorapprox} suggests that when $\varsigma$ is sufficiently small, the change in $\rhom^\star(\varsigma^2)$ is driven primarily by the behavior of the PDFs at the boundary between the domains of the partition.
\end{remark}

\begin{example}[Competition of Scales]
Let $r\in \mathbb R$, fix $q_1=q_2=0.5$, and define 
\begin{subequations}
\begin{align}
p_1(r)= (2\pi)^{-1/2} \exp[-(r+1)^2/2], \\
p_2(r)= (2\pi)^{-1/2} \exp[-(r-1)^2/2].
\end{align}
\end{subequations}
By symmetry, the optimal partition corresponds to $D_1=(-\infty,0)$ and $D_2=(0,\infty)$, with the point $r=0$ being assigned to either set.  One also finds that
\begin{align}
\rhom^\star(0) = \int_{0}^{\infty} \d r \,\,\, \frac{1}{\sqrt{2\pi}} e^{-\frac{(r+1)^2}{2}} < 0.5.
\end{align}
By known properties of the Gaussian distribution, one also finds that
\begin{align}
\rhom^\star(\varsigma^2) = \int_{0}^{\infty} \d r \,\,\, \frac{1}{\sqrt{2\pi(1+\varsigma^2)}} e^{-\frac{(r+1)^2}{2(1+\varsigma^2)}}.
\end{align}
Note that the inequalities $\rhom^\star(0) < \rhom^\star(\varsigma^2) < 1/2$ always hold, although 
\begin{subequations}
\begin{align}
\lim_{\varsigma^2 \to 0} \rhom^\star(\varsigma^2) &= \rhom^\star(0), \label{eq:lim1} \\
\lim_{\varsigma^2\to \infty} \rhom^\star(\varsigma^2) &= 1/2. \label{eq:lim2}
\end{align}
\end{subequations}
\end{example}

\begin{remark}
Equations \eqref{eq:lim1} and \eqref{eq:lim2} tell us that $\varsigma$ competes with the scale of the PDFs and/or their supports to determine the extent to which noise impacts $\rhom$.  The first case states the sufficiently small noise has essentially no impact.  The second case states that large noise reverts a classifier to the limit of its degeneracy, so that property (P3) is no longer valid.  This example illustrates the intuitive idea that too much noise renders a classifier meaningless.  Note that Eq.\ \eqref{eq:varinequality} also diverges in this case.  
\end{remark}

\subsection{Gershgorin Radius as an Objective Function for Assay Optimization}

Equations \eqref{eq:ebound} and \eqref{eq:varinequality}: (I) are uniform bounds on the classification error and prevalence uncertainty; and (II) only depend on the assay through the Gershgorin radius $\rhom$.  These features  imply that $\rhom$ is a valid objective function whose minimum quantifies uncertainties of interest.  Given a choice among assays (or more generally, input spaces $\Gamma$), we can identify the ``best'' one as that whose minimum value of $\rhom$ is smallest.

The utility of inequality \eqref{eq:varinequality} is rooted in our ability to numerically solve this optimization problem.  Section \ref{sec:minrho} accomplishes this for binary classification under the assumption that the $p_j(\br)$ are known. In practice, we are only given empirical data drawn from these distributions.  In such cases, statistical modeling of the conditional PDFs may be prohibitive if the points $\br$ are high-dimensional and/or there is insufficient data to account for correlations between different dimensions. 

Several empirical methods exist to address this problem.  Clustering algorithms can be used to approximate classification domains, for which we can define the empirical estimates
\begin{align}
P_{j,k} \approx \tilde P_{j,k} = \frac{1}{m_k}\sum_{i=1}^{m_k}\mathbb I(\br_{i,k}\in D_j)
\end{align}
where $\br_{i,k}$ is the $i$th point from the $k$th class, $m_k$ is the number of training samples in class $k$, and $D_j$ is the domain associated with the $j$th cluster.  More recently, in a binary setting, we showed that efficient homotopy methods can be used to minimize the misclassification rate of empirical training data without needing to assume the full conditional PDFs \cite{PartI}.  However, we acknowledge that efficient optimization of $\rhom$ remains an interesting problem.  See Refs.\ \cite{Patrone21,Patrone22_1,RW,SmithUQ} for further discussion of such issues.

\subsection{Additional Open Directions}

Uniform bounds on assay performance are useful in settings where it is difficult to predict how prevalence will change.    However, slowly spreading diseases may settle into an endemic state characterized by a time-invariant prevalence.  In such cases, uniform bounds may overestimate the uncertainty associated with a diagnostic assay, leading to the need to determine whether tighter, prevalence-dependent bounds on uncertainty can be deduced, especially in a multi-class setting.  


{\it Acknowledgements:} This work is a contribution of the National Institutes of Standards and Technology and is therefore not subject to copyright in the United States.  We thank Dr.\ A.\ Krishnaswamy-Usha and Dr.\ M.\ Donahue for helpful discussion during preparation of this manuscript.  


\bibliographystyle{elsarticle-num}
\bibliography{Assay_Inequalities}

\begin{thebibliography}{10}
\expandafter\ifx\csname url\endcsname\relax
  \def\url#1{\texttt{#1}}\fi
\expandafter\ifx\csname urlprefix\endcsname\relax\def\urlprefix{URL }\fi
\expandafter\ifx\csname href\endcsname\relax
  \def\href#1#2{#2} \def\path#1{#1}\fi

\bibitem{assaynumber}
G.~Liu, J.~F. Rusling, Covid-19 antibody tests and their limitations, ACS
  Sensors 6~(3) (2021) 593--612.

\bibitem{EUA}
FDA, Eua authorized serology test performance,
  {https://www.fda.gov/medical-devices/coronavirus-disease-2019-covid-19-emergency-use-authorizations-medical-devices/eua-authorized-serology-test-performance},
  accessed: 2020-09-16 (2020).

\bibitem{revoked}
R.~Colgrove, L.~A. Bruno-Murtha, C.~A. Chastain, K.~E. Hanson, F.~Lee, A.~R.
  Odom~John, R.~Humphries, {Tale of the Titers: Serologic Testing for
  SARS-CoV-2—Yes, No, and Maybe, With Clinical Examples From the IDSA
  Diagnostics Committee}, Open Forum Infectious Diseases 10~(1) (12 2022).

\bibitem{controversial}
E.~Bendavid, B.~Mulaney, N.~Sood, S.~Shah, R.~Bromley-Dulfano, C.~Lai,
  Z.~Weissberg, R.~Saavedra-Walker, J.~Tedrow, A.~Bogan, T.~Kupiec, D.~Eichner,
  R.~Gupta, J.~P.~A. Ioannidis, J.~Bhattacharya, {COVID-19 antibody
  seroprevalence in Santa Clara County, California}, International Journal of
  Epidemiology 50~(2) (2021) 410--419.

\bibitem{flawed}
J.~Abbasi, {The Flawed Science of Antibody Testing for SARS-CoV-2 Immunity},
  JAMA 326~(18) (2021) 1781--1782.
\newblock \href {https://doi.org/10.1001/jama.2021.18919}
  {\path{doi:10.1001/jama.2021.18919}}.

\bibitem{Medstat1}
R.~Dorfman, \href{http://www.jstor.org/stable/2235930}{The detection of
  defective members of large populations}, The Annals of Mathematical
  Statistics 14~(4) (1943) 436--440.
\newline\urlprefix\url{http://www.jstor.org/stable/2235930}

\bibitem{Lewis12}
F.~I. Lewis, P.~R. Torgerson, \href{https://doi.org/10.1186/1742-7622-9-9}{A
  tutorial in estimating the prevalence of disease in humans and animals in the
  absence of a gold standard diagnostic}, Emerging Themes in Epidemiology 9~(1)
  (2012) 9.
\newblock \href {https://doi.org/10.1186/1742-7622-9-9}
  {\path{doi:10.1186/1742-7622-9-9}}.
\newline\urlprefix\url{https://doi.org/10.1186/1742-7622-9-9}

\bibitem{Lew89}
R.~A. Lew, P.~S. Levy,
  \href{https://onlinelibrary.wiley.com/doi/abs/10.1002/sim.4780081006}{Estimation
  of prevalence on the basis of screening tests}, Statistics in Medicine 8~(10)
  (1989) 1225--1230.
\newblock \href
  {http://arxiv.org/abs/https://onlinelibrary.wiley.com/doi/pdf/10.1002/sim.4780081006}
  {\path{arXiv:https://onlinelibrary.wiley.com/doi/pdf/10.1002/sim.4780081006}},
  \href {https://doi.org/https://doi.org/10.1002/sim.4780081006}
  {\path{doi:https://doi.org/10.1002/sim.4780081006}}.
\newline\urlprefix\url{https://onlinelibrary.wiley.com/doi/abs/10.1002/sim.4780081006}

\bibitem{Lang14}
Z.~Lang, J.~Reiczigel,
  \href{https://www.sciencedirect.com/science/article/pii/S0167587713002936}{Confidence
  limits for prevalence of disease adjusted for estimated sensitivity and
  specificity}, Preventive Veterinary Medicine 113~(1) (2014) 13--22.
\newblock \href
  {https://doi.org/https://doi.org/10.1016/j.prevetmed.2013.09.015}
  {\path{doi:https://doi.org/10.1016/j.prevetmed.2013.09.015}}.
\newline\urlprefix\url{https://www.sciencedirect.com/science/article/pii/S0167587713002936}

\bibitem{OldPrevOpt}
C.~Brownie, J.-P. Habicht, \href{http://www.jstor.org/stable/2530910}{Selecting
  a screening cut-off point or diagnostic criterion for comparing prevalences
  of disease}, Biometrics 40~(3) (1984) 675--684.
\newline\urlprefix\url{http://www.jstor.org/stable/2530910}

\bibitem{Qiu19}
Z.~Qiu, L.~Peng, A.~Manatunga, Y.~Guo,
  \href{https://www.sciencedirect.com/science/article/pii/S0167947318302779}{A
  smooth nonparametric approach to determining cut-points of a continuous
  scale}, Computational Statistics \& Data Analysis 134 (2019) 186--210.
\newblock \href {https://doi.org/https://doi.org/10.1016/j.csda.2018.11.001}
  {\path{doi:https://doi.org/10.1016/j.csda.2018.11.001}}.
\newline\urlprefix\url{https://www.sciencedirect.com/science/article/pii/S0167947318302779}

\bibitem{Patrone21}
P.~N. Patrone, A.~J. Kearsley, Classification under uncertainty: data analysis
  for diagnostic antibody testing., Mathematical medicine and biology : a
  journal of the IMA (2021).

\bibitem{Patrone22}
P.~N. Patrone, P.~Bedekar, N.~Pisanic, Y.~C. Manabe, D.~L. Thomas, C.~D.
  Heaney, A.~J. Kearsley,
  \href{https://www.sciencedirect.com/science/article/pii/S0025556422000608}{Optimal
  decision theory for diagnostic testing: Minimizing indeterminate classes with
  applications to saliva-based sars-cov-2 antibody assays}, Mathematical
  Biosciences 351 (2022) 108858.
\newblock \href {https://doi.org/https://doi.org/10.1016/j.mbs.2022.108858}
  {\path{doi:https://doi.org/10.1016/j.mbs.2022.108858}}.
\newline\urlprefix\url{https://www.sciencedirect.com/science/article/pii/S0025556422000608}

\bibitem{Patrone22_1}
P.~Patrone, A.~Kearsley, \href{https://arxiv.org/abs/2203.12792}{Minimizing
  uncertainty in prevalence estimates} (2022).
\newblock \href {https://doi.org/10.48550/ARXIV.2203.12792}
  {\path{doi:10.48550/ARXIV.2203.12792}}.
\newline\urlprefix\url{https://arxiv.org/abs/2203.12792}

\bibitem{Luke22}
R.~A. Luke, A.~J. Kearsley, N.~Pisanic, Y.~C. Manabe, D.~L. Thomas, C.~D.
  Heaney, P.~N. Patrone, \href{https://arxiv.org/abs/2206.14316}{Modeling in
  higher dimensions to improve diagnostic testing accuracy: theory and examples
  for multiplex saliva-based sars-cov-2 antibody assays} (2022).
\newblock \href {https://doi.org/10.48550/ARXIV.2206.14316}
  {\path{doi:10.48550/ARXIV.2206.14316}}.
\newline\urlprefix\url{https://arxiv.org/abs/2206.14316}

\bibitem{Luke23}
R.~A. Luke, A.~J. Kearsley, P.~N. Patrone,
  \href{https://www.sciencedirect.com/science/article/pii/S0025556423000238}{Optimal
  classification and generalized prevalence estimates for diagnostic settings
  with more than two classes}, Mathematical Biosciences 358 (2023) 108982.
\newblock \href {https://doi.org/https://doi.org/10.1016/j.mbs.2023.108982}
  {\path{doi:https://doi.org/10.1016/j.mbs.2023.108982}}.
\newline\urlprefix\url{https://www.sciencedirect.com/science/article/pii/S0025556423000238}

\bibitem{Bedekar22}
P.~Bedekar, A.~J. Kearsley, P.~N. Patrone,
  \href{https://www.sciencedirect.com/science/article/pii/S0022519322003666}{Prevalence
  estimation and optimal classification methods to account for time dependence
  in antibody levels}, Journal of Theoretical Biology 559 (2023) 111375.
\newblock \href {https://doi.org/https://doi.org/10.1016/j.jtbi.2022.111375}
  {\path{doi:https://doi.org/10.1016/j.jtbi.2022.111375}}.
\newline\urlprefix\url{https://www.sciencedirect.com/science/article/pii/S0022519322003666}

\bibitem{PartI}
P.~N. Patrone, R.~A. Binder, C.~S. Forconi, A.~M. Moormann, A.~J. Kearsley,
  \href{https://arxiv.org/abs/2309.00645}{Analysis of diagnostics (part i):
  Prevalence, uncertainty quantification, and machine learning} (2024).
\newblock \href {http://arxiv.org/abs/2309.00645} {\path{arXiv:2309.00645}}.
\newline\urlprefix\url{https://arxiv.org/abs/2309.00645}

\bibitem{change}
M.~Lisboa~Bastos, G.~Tavaziva, S.~K. Abidi, J.~R. Campbell, L.-P. Haraoui,
  J.~C. Johnston, Z.~Lan, S.~Law, E.~MacLean, A.~Trajman, D.~Menzies,
  A.~Benedetti, F.~Ahmad~Khan, Diagnostic accuracy of serological tests for
  covid-19: systematic review and meta-analysis, BMJ 370 (2020).
\newblock \href {https://doi.org/10.1136/bmj.m2516}
  {\path{doi:10.1136/bmj.m2516}}.

\bibitem{Metrics1}
K.~Dembczy{\'{n}}ski, W.~Kot{\l}owski, O.~Koyejo, N.~Natarajan,
  \href{https://proceedings.mlr.press/v70/dembczynski17a.html}{Consistency
  analysis for binary classification revisited}, in: D.~Precup, Y.~W. Teh
  (Eds.), Proceedings of the 34th International Conference on Machine Learning,
  Vol.~70 of Proceedings of Machine Learning Research, PMLR, 2017, pp.
  961--969.
\newline\urlprefix\url{https://proceedings.mlr.press/v70/dembczynski17a.html}

\bibitem{Metrics2}
H.~Narasimhan, H.~Ramaswamy, A.~Saha, S.~Agarwal,
  \href{https://proceedings.mlr.press/v37/narasimhanb15.html}{Consistent
  multiclass algorithms for complex performance measures}, in: F.~Bach, D.~Blei
  (Eds.), Proceedings of the 32nd International Conference on Machine Learning,
  Vol.~37 of Proceedings of Machine Learning Research, PMLR, Lille, France,
  2015, pp. 2398--2407.
\newline\urlprefix\url{https://proceedings.mlr.press/v37/narasimhanb15.html}

\bibitem{confusion}
P.~Machart, L.~Ralaivola, Confusion matrix stability bounds for multiclass
  classification (2012).
\newblock \href {http://arxiv.org/abs/1202.6221} {\path{arXiv:1202.6221}}.

\bibitem{Imbalanced}
A.~Menon, H.~Narasimhan, S.~Agarwal, S.~Chawla,
  \href{https://proceedings.mlr.press/v28/menon13a.html}{On the statistical
  consistency of algorithms for binary classification under class imbalance},
  in: S.~Dasgupta, D.~McAllester (Eds.), Proceedings of the 30th International
  Conference on Machine Learning, Vol.~28 of Proceedings of Machine Learning
  Research, PMLR, Atlanta, Georgia, USA, 2013, pp. 603--611.
\newline\urlprefix\url{https://proceedings.mlr.press/v28/menon13a.html}

\bibitem{Imbalanced2}
M.~Steininger, K.~Kobs, P.~Davidson, A.~Krause, A.~Hotho,
  \href{https://doi.org/10.1007/s10994-021-06023-5}{Density-based weighting for
  imbalanced regression}, Machine Learning 110~(8) (2021) 2187--2211.
\newblock \href {https://doi.org/10.1007/s10994-021-06023-5}
  {\path{doi:10.1007/s10994-021-06023-5}}.
\newline\urlprefix\url{https://doi.org/10.1007/s10994-021-06023-5}

\bibitem{Metrics3}
O.~Koyejo, N.~Natarajan, P.~Ravikumar, I.~S. Dhillon,
  \href{https://api.semanticscholar.org/CorpusID:6680925}{Consistent binary
  classification with generalized performance metrics}, in: Neural Information
  Processing Systems, 2014.
\newline\urlprefix\url{https://api.semanticscholar.org/CorpusID:6680925}

\bibitem{Metrics4}
O.~Koyejo, N.~Natarajan, P.~Ravikumar, I.~S. Dhillon,
  \href{https://api.semanticscholar.org/CorpusID:6348016}{Consistent multilabel
  classification}, in: Neural Information Processing Systems, 2015.
\newline\urlprefix\url{https://api.semanticscholar.org/CorpusID:6348016}

\bibitem{Metrics5}
H.~Narasimhan, P.~Kar, P.~Jain,
  \href{https://api.semanticscholar.org/CorpusID:1424505}{Optimizing
  non-decomposable performance measures: A tale of two classes}, in:
  International Conference on Machine Learning, 2015.
\newline\urlprefix\url{https://api.semanticscholar.org/CorpusID:1424505}

\bibitem{Metrics6}
A.~S. Jadhav,
  \href{https://www.sciencedirect.com/science/article/pii/S0957417420302153}{A
  novel weighted tpr-tnr measure to assess performance of the classifiers},
  Expert Systems with Applications 152 (2020) 113391.
\newblock \href {https://doi.org/https://doi.org/10.1016/j.eswa.2020.113391}
  {\path{doi:https://doi.org/10.1016/j.eswa.2020.113391}}.
\newline\urlprefix\url{https://www.sciencedirect.com/science/article/pii/S0957417420302153}

\bibitem{multi}
M.~Evans, N.~Hastings, B.~Peacock, C.~Forbes, Statistical Distributions, Wiley,
  2011.

\bibitem{pspace}
D.~Stroock, Probability Theory: An Analytic View, Cambridge University Press,
  2010.

\bibitem{dim1}
A.~Hachim, N.~Kavian, C.~A. Cohen, A.~W.~H. Chin, D.~K.~W. Chu, C.~K.~P. Mok,
  O.~T.~Y. Tsang, Y.~C. Yeung, R.~A. P.~M. Perera, L.~L.~M. Poon, J.~S.~M.
  Peiris, S.~A. Valkenburg, Orf8 and orf3b antibodies are accurate serological
  markers of early and late sars-cov-2 infection, Nature Immunology 21~(10)
  (2020) 1293--1301.
\newblock \href {https://doi.org/10.1038/s41590-020-0773-7}
  {\path{doi:10.1038/s41590-020-0773-7}}.

\bibitem{dim2}
L.~Grzelak, S.~Temmam, C.~Planchais, C.~Demeret, L.~Tondeur, C.~Huon,
  F.~Guivel-Benhassine, I.~Staropoli, M.~Chazal, J.~Dufloo, D.~Planas,
  J.~Buchrieser, M.~M. Rajah, R.~Robinot, F.~Porrot, M.~Albert, K.-Y. Chen,
  B.~Crescenzo-Chaigne, F.~Donati, F.~Anna, P.~Souque, M.~Gransagne,
  J.~Bellalou, M.~Nowakowski, M.~Backovic, L.~Bouadma, L.~Le~Fevre,
  Q.~Le~Hingrat, D.~Descamps, A.~Pourbaix, C.~Laou{\'e}nan, J.~Ghosn,
  Y.~Yazdanpanah, C.~Besombes, N.~Jolly, S.~Pellerin-Fernandes, O.~Cheny, M.-N.
  Ungeheuer, G.~Mellon, P.~Morel, S.~Rolland, F.~A. Rey, S.~Behillil, V.~Enouf,
  A.~Lemaitre, M.-A. Cr{\'e}ach, S.~Petres, N.~Escriou, P.~Charneau,
  A.~Fontanet, B.~Hoen, T.~Bruel, M.~Eloit, H.~Mouquet, O.~Schwartz, S.~van~der
  Werf, A comparison of four serological assays for detecting
  anti{\textendash}sars-cov-2 antibodies in human serum samples from different
  populations 12~(559) (2020).
\newblock \href {https://doi.org/10.1126/scitranslmed.abc3103}
  {\path{doi:10.1126/scitranslmed.abc3103}}.

\bibitem{dim3}
A.~Algaissi, M.~A. Alfaleh, S.~Hala, T.~S. Abujamel, S.~S. Alamri, S.~A.
  Almahboub, K.~A. Alluhaybi, H.~I. Hobani, R.~M. Alsulaiman, R.~H. AlHarbi,
  M.-Z. ElAssouli, R.~Y. Alhabbab, A.~A. AlSaieedi, W.~H. Abdulaal, A.~A.
  Al-Somali, F.~S. Alofi, A.~A. Khogeer, A.~A. Alkayyal, A.~B. Mahmoud,
  N.~A.~M. Almontashiri, A.~Pain, A.~M. Hashem, Sars-cov-2 s1 and n-based
  serological assays reveal rapid seroconversion and induction of specific
  antibody response in covid-19 patients, Scientific Reports 10~(1) (2020)
  16561.
\newblock \href {https://doi.org/10.1038/s41598-020-73491-5}
  {\path{doi:10.1038/s41598-020-73491-5}}.

\bibitem{dim4}
R.~A. Binder, G.~F. Fujimori, C.~S. Forconi, G.~W. Reed, L.~S. Silva, P.~S.
  Lakshmi, A.~Higgins, L.~Cincotta, P.~Dutta, M.-C. Salive, V.~Mangolds,
  O.~Anya, J.~M. Calvo~Calle, T.~Nixon, Q.~Tang, M.~Wessolossky, Y.~Wang, D.~A.
  Ritacco, C.~S. Bly, S.~Fischinger, C.~Atyeo, P.~O. Oluoch, B.~Odwar, J.~A.
  Bailey, A.~Maldonado-Contreras, J.~P. Haran, A.~G. Schmidt, L.~Cavacini,
  G.~Alter, A.~M. Moormann, {SARS-CoV-2 Serosurveys: How Antigen, Isotype and
  Threshold Choices Affect the Outcome}, The Journal of Infectious Diseases
  227~(3) (2022) 371--380.

\bibitem{dim5}
N.~Pisanic, P.~R. Randad, K.~Kruczynski, Y.~C. Manabe, D.~L. Thomas, A.~Pekosz,
  S.~L. Klein, M.~J. Betenbaugh, W.~A. Clarke, O.~Laeyendecker, P.~P.
  Caturegli, H.~B. Larman, B.~Detrick, J.~K. Fairley, A.~C. Sherman,
  N.~Rouphael, S.~Edupuganti, D.~A. Granger, S.~W. Granger, M.~H. Collins,
  C.~D. Heaney, M.~J. Loeffelholz, Covid-19 serology at population scale:
  Sars-cov-2-specific antibody responses in saliva, Journal of Clinical
  Microbiology 59~(1) (2020) e02204--20.
\newblock \href {https://doi.org/10.1128/JCM.02204-20}
  {\path{doi:10.1128/JCM.02204-20}}.

\bibitem{dim6}
P.~R. Randad, N.~Pisanic, K.~Kruczynski, T.~Howard, M.~G. Rivera, K.~Spicer,
  A.~A. Antar, T.~Penson, D.~L. Thomas, A.~Pekosz, N.~Ndahiro, L.~Aliyu, M.~J.
  Betenbaugh, H.~Manley, B.~Detrick, M.~Katz, S.~Cosgrove, C.~Rock, I.~Zyskind,
  J.~I. Silverberg, A.~Z. Rosenberg, P.~Duggal, Y.~C. Manabe, M.~H. Collins,
  C.~D. Heaney, Durability of sars-cov-2-specific igg responses in saliva for
  up to 8 months after infection (2021).

\bibitem{dim7}
C.~D. Heaney, N.~Pisanic, P.~R. Randad, K.~Kruczynski, T.~Howard, X.~Zhu,
  K.~Littlefield, E.~U. Patel, R.~Shrestha, O.~Laeyendecker, S.~Shoham,
  D.~Sullivan, K.~Gebo, D.~Hanley, A.~D. Redd, T.~C. Quinn, A.~Casadevall,
  J.~M. Zenilman, A.~Pekosz, E.~M. Bloch, A.~A. Tobian, Comparative performance
  of multiplex salivary and commercially available serologic assays to detect
  sars-cov-2 igg and neutralization titers, Journal of Clinical Virology 145
  (2021) 104997.

\bibitem{multiset}
J.~Hein, \href{https://books.google.com/books?id=qLNfpb1hBWwC}{Discrete
  Mathematics}, Discrete Mathematics and Logic Series, Jones and Bartlett
  Publishers, 2003.
\newline\urlprefix\url{https://books.google.com/books?id=qLNfpb1hBWwC}

\bibitem{Zhang}
T.~Zhang, \href{http://www.jstor.org/stable/3448494}{Statistical behavior and
  consistency of classification methods based on convex risk minimization}, The
  Annals of Statistics 32~(1) (2004) 56--85.
\newline\urlprefix\url{http://www.jstor.org/stable/3448494}

\bibitem{RW}
C.~Rasmussen, C.~Williams, Gaussian Processes for Machine Learning, Adaptative
  computation and machine learning series, University Press Group Limited,
  2006.

\bibitem{Tao}
T.~Tao, An Introduction to Measure Theory, Graduate studies in mathematics,
  American Mathematical Society, 2013.

\bibitem{Gershgorin1}
S.~A. Gershgorin, {\"U}ber die {Abgrenzung} der {Eigenwerte} einer {Matrix}.,
  Bull. Acad. Sci. URSS 1931~(6) (1931) 749--754.

\bibitem{Gershgorin2}
R.~A. Horn, C.~R. Johnson, Matrix Analysis, 2nd Edition, Cambridge University
  Press, 2012.
\newblock \href {https://doi.org/10.1017/CBO9781139020411}
  {\path{doi:10.1017/CBO9781139020411}}.

\bibitem{StochasticMatrix1}
P.~Gagniuc, Markov Chains: From Theory to Implementation and Experimentation,
  2017.
\newblock \href {https://doi.org/10.1002/9781119387596}
  {\path{doi:10.1002/9781119387596}}.

\bibitem{grad}
D.~Zwillinger, A.~Jeffrey, Table of Integrals, Series, and Products, Elsevier
  Science, 2000.

\bibitem{Ideal}
E.~Versi, "gold standard" is an appropriate term., BMJ 305~(6846) (1992)
  187--187.
\newblock \href {https://doi.org/10.1136/bmj.305.6846.187-b}
  {\path{doi:10.1136/bmj.305.6846.187-b}}.

\bibitem{montecarlo}
R.~E. Caflisch, Monte carlo and quasi-monte carlo methods, Acta Numerica 7
  (1998) 1–49.
\newblock \href {https://doi.org/10.1017/S0962492900002804}
  {\path{doi:10.1017/S0962492900002804}}.

\bibitem{Totvar}
S.~M. Ross, 3 - conditional probability and conditional expectation, in: S.~M.
  Ross (Ed.), Introduction to Probability Models (Thirteenth Edition),
  thirteenth edition Edition, Academic Press, 2024, pp. 103--200.
\newblock \href
  {https://doi.org/https://doi.org/10.1016/B978-0-44-318761-2.00008-7}
  {\path{doi:https://doi.org/10.1016/B978-0-44-318761-2.00008-7}}.

\bibitem{Krey}
E.~Kreyszig, Introductory Functional Analysis with Applications, Wiley Classics
  Library, Wiley, 1991.

\bibitem{bathtub}
E.~Lieb, M.~Loss, A.~M. Society, Analysis, Crm Proceedings \& Lecture Notes,
  American Mathematical Society, 2001.

\bibitem{toeplitz}
A.~B{\"o}ttcher, S.~Grudsky,
  \href{https://books.google.com/books?id=Dmr0BwAAQBAJ}{Toeplitz Matrices,
  Asymptotic Linear Algebra, and Functional Analysis}, Birkh{\"a}user Basel,
  2012.
\newline\urlprefix\url{https://books.google.com/books?id=Dmr0BwAAQBAJ}

\bibitem{PM}
C.~Miranda, Un'osservazione su un teorema di Brouwer, Pubblicazioni (Istituto
  per le applicazioni del calcolo), Consiglio Nazionale delle Ricerche, 1940.

\bibitem{Noisy_labels}
N.~Natarajan, I.~S. Dhillon, P.~Ravikumar, A.~Tewari,
  \href{http://jmlr.org/papers/v18/15-226.html}{Cost-sensitive learning with
  noisy labels}, Journal of Machine Learning Research 18~(155) (2018) 1--33.
\newline\urlprefix\url{http://jmlr.org/papers/v18/15-226.html}

\bibitem{Divisible1}
F.~Steutel,
  \href{https://www.sciencedirect.com/science/article/pii/0304414973900082}{Some
  recent results in infinite divisibility}, Stochastic Processes and their
  Applications 1~(2) (1973) 125--143.
\newblock \href {https://doi.org/https://doi.org/10.1016/0304-4149(73)90008-2}
  {\path{doi:https://doi.org/10.1016/0304-4149(73)90008-2}}.
\newline\urlprefix\url{https://www.sciencedirect.com/science/article/pii/0304414973900082}

\bibitem{Divisible2}
A.~Bose, A.~Dasgupta, H.~Rubin, \href{http://www.jstor.org/stable/25051430}{A
  contemporary review and bibliography of infinitely divisible distributions
  and processes}, Sankhyā: The Indian Journal of Statistics, Series A
  (1961-2002) 64~(3) (2002) 763--819.
\newline\urlprefix\url{http://www.jstor.org/stable/25051430}

\bibitem{SmithUQ}
R.~Smith, \href{https://books.google.com/books?id=4c1GAgAAQBAJ}{Uncertainty
  Quantification: Theory, Implementation, and Applications}, Computational
  Science and Engineering, Society for Industrial and Applied Mathematics,
  2013.
\newline\urlprefix\url{https://books.google.com/books?id=4c1GAgAAQBAJ}

\end{thebibliography}

\end{document}